\documentclass[twoside,11pt]{article}

\usepackage{xparse}
\usepackage{ifthen}
\usepackage[utf8]{inputenc}\DeclareUnicodeCharacter{2212}{-}
\usepackage[T1]{fontenc}

\usepackage{amsmath,amssymb,amsfonts}
\usepackage{mathtools}
\usepackage{dsfont}
\usepackage{amsthm}
\usepackage{thmtools, thm-restate}

\usepackage{pgfplots}
\usepackage{pgfplotstable}
\usetikzlibrary{spy}
\usepgfplotslibrary{groupplots}
\pgfplotsset{compat=1.16}

\usepackage{booktabs}
\usepackage[section]{placeins}
\usepackage{wrapfig}

\usepackage{siunitx}
\usepackage{enumitem}
\usepackage{xspace}
\usepackage{csquotes}
\usepackage{multicol}
\usepackage{lscape}
\usepackage{layouts}
\usepackage{csquotes}

\DeclareMathOperator{\expectsymb}{\mathds{E}}
\DeclareMathOperator{\varsymb}{\mathds{V}}
\DeclareMathOperator*{\argmin}{\mathrm{argmin}}

\newcommand{\indsymb}{\mathds{1}}
\newcommand{\probsymb}{\mathds{P}}
\newcommand{\real}{\mathds{R}}
\newcommand{\natnum}{\mathds{N}}

\let\set\undefined
\NewDocumentCommand{\set}{m}{\left\{#1\right\}}
\NewDocumentCommand{\abs}{m}{\left|#1\right|}

\NewDocumentCommand{\wrapbrackets}{m}{\left[ #1 \right]}

\NewDocumentCommand{\expect}{e{_}o}{
    \def\decorated{\expectsymb\IfNoValueTF{#1}{}{_{#1}}}
    \IfNoValueTF{#2}{\decorated}{\decorated\!\wrapbrackets{#2}}
}
\NewDocumentCommand{\variance}{e{_}o}{
    \def\decorated{\varsymb\IfNoValueTF{#1}{}{_{#1}}}
    \IfNoValueTF{#2}{\decorated}{\decorated\!\wrapbrackets{#2}}
}
\NewDocumentCommand{\trisk}{e{_}o}{
    \def\decorated{\operatorname{R}\!{}^*\IfNoValueTF{#1}{}{_{#1}}}
    \IfNoValueTF{#2}{\decorated}{\decorated\!\wrapbrackets{#2}}
}

\NewDocumentCommand{\terisk}{e{_}o}{
    \def\decorated{\operatorname{\hat{R}}\!{}^*\IfNoValueTF{#1}{}{_{#1}}}
    \IfNoValueTF{#2}{\decorated}{\decorated\!\wrapbrackets{#2}}
}

\NewDocumentCommand{\orisk}{e{_^}o}{
    \def\decorated{\operatorname{R}\IfNoValueTF{#1}{}{_{#1}}\IfNoValueTF{#2}{}{^{#2}}}
    \IfNoValueTF{#3}{\decorated}{\decorated\!\wrapbrackets{#3}}
}

\NewDocumentCommand{\oerisk}{e{_^}o}{
    \def\decorated{\operatorname{\hat{R}}\IfNoValueTF{#1}{}{_{#1}}\IfNoValueTF{#2}{}{^{#2}}}
    \IfNoValueTF{#3}{\decorated}{\decorated\!\wrapbrackets{#3}}
}

\NewDocumentCommand{\ind}{o}{
    \IfNoValueTF{#1}{\indsymb}{\indsymb\!\set{#1}}
}
\NewDocumentCommand{\prob}{o}{
    \IfNoValueTF{#1}{\probsymb}{\probsymb\!\set{#1}}
}
\NewDocumentCommand{\cprob}{mm}{
    \prob[#1\mid#2]
}
\NewDocumentCommand{\defmap}{mmm}{\ensuremath{#1: #2 \longrightarrow #3}}
\NewDocumentCommand{\vect}{m}{\mathbf{#1}}
\NewDocumentCommand{\indvect}{m}{\vect{1}^{\!(#1)}}
\NewDocumentCommand{\allmaps}{mm}{\mathcal{M}(#1, #2)}

\NewDocumentCommand{\pswt}{o}{
    \mathfrak{P}_{\IfNoValueTF{#1}{p}{p_{#1}}}
}
\newcommand{\rademacher}{\operatorname{\mathfrak{R}}}
\newcommand{\lossfn}{f}
\newcommand{\opl}{\lossfn^*_{+}}
\newcommand{\oml}{\lossfn^*_{-}}
\newcommand{\cpl}{\lossfn_{+}}
\newcommand{\cml}{\lossfn_{-}}

\newcommand{\numlabels}{l}
\newcommand{\propensity}{p}
\newcommand{\dataspace}{\mathcal{X}}
\newcommand{\datapvar}{X}

\newcommand{\labelspace}{\mathcal{Y}}
\newcommand{\truelabels}{Y^*}
\newcommand{\obslabels}{Y}
\newcommand{\predlabels}{\hat{y}}

\newcommand{\predictspace}{\hat{\labelspace}}
\newcommand{\functionclass}{\mathcal{H}}

\pgfplotstableset{
BasicTableStyle/.style={
    every head row/.style={
        before row={\toprule},
        after row={\midrule}
    },
    every last row/.style={after row=\bottomrule}
}}

\makeatletter
\DeclareRobustCommand\onedot{\futurelet\@let@token\@onedot}
\def\@onedot{\ifx\@let@token.\else.\null\fi\xspace}
\newcommand{\ie}{i.e\onedot}
\newcommand{\eg}{e.g\onedot}

\newcommand{\cf}{cf\onedot}

\makeatother

\usepackage{jmlr2e}

\ShortHeadings{Unbiased Loss Functions for Evaluation and Training with Missing Labels}{Schultheis and Babbar}
\firstpageno{1}

\begin{document}

\title{Unbiased Loss Functions for Multilabel Classification with Missing Labels}

\author{\name Erik Schultheis \email firstname.lastname@aalto.fi \\
       \addr Dept. of Computer Science
       at Aalto University\\
       Konemiehentie 2, 02150 Espoo
       \AND
       \name Rohit Babbar \email firstname.lastname@aalto.fi \\
       \addr Dept. of Computer Science
       at Aalto University\\
       Konemiehentie 2, 02150 Espoo}

\editor{}

\maketitle

\begin{abstract}%
  This paper considers binary and multilabel classification problems in a
  setting where labels are missing independently and with a known rate.
  Missing labels are a ubiquitous phenomenon in extreme multi-label
  classification (XMC) tasks, such as matching Wikipedia articles to a small
  subset out of the hundreds of thousands of possible tags, where no human
  annotator can possibly check the validity of all the negative samples. For
  this reason, propensity-scored precision -- an unbiased estimate for
  precision-at-k under a known noise model -- has become one of the standard
  metrics in XMC. Few methods take this problem into account already during
  the training phase, and all are limited to loss functions that can be
  decomposed into a sum of contributions from each individual label. A typical
  approach to training is to reduce the multilabel problem into a series of
  binary or multiclass problems, and it has been shown that if the surrogate
  task should be consistent for optimizing recall, the resulting loss function
  is not decomposable over labels. Therefore, this paper derives the unique unbiased estimators for the different multilabel reductions, including the
  non-decomposable ones. These estimators suffer from increased variance and may
  lead to ill-posed optimization problems, which we address by switching to
  convex upper-bounds. The theoretical considerations are further supplemented
  by an experimental study showing that the switch to unbiased estimators
  significantly alters the bias-variance trade-off and may thus require
  stronger regularization, which in some cases can negate the
  benefits of unbiased estimation.


\end{abstract}


\section{Introduction}
Extreme multilabel classification (XMC) is a machine learning setting in which
the goal is to predict a small subset of positive (or relevant) labels for
each data instance out of a very large (thousands to millions) set of possible
labels. Such problems arise for example when annotating large encyclopedia
\citep{dekel2010multiclass, partalas2015lshtc}, in fine-grained image
classification \citep{deng2010does}, and next-word prediction
\citep{mikolov2013efficient}. Further applications of XMC are recommendation
systems, web-advertising and prediction of related searches in a search engine
\citep{Agrawal13,prabhu2014fastxml,jain2019slice,Dahiya21b}.

Typical datasets in these scenarios are very large, resulting in possibly
billions of (data, label) pairs. This means that it is not possible for human
annotators to check each pair, and thus the available training data is likely
to contain some errors. Even annotating only a few samples in order to
generate a clean test set can be prohibitively expensive. Fortunately, in many
cases it is possible to constrain the structure of the labeling errors.
Consider, for example, the case of tagging documents. Here, we can assume that
each label with which the document has been tagged has been deemed relevant by
the annotator, and thus is relatively surely a correct label. On the other
hand, the annotator cannot possibly check hundreds of thousands of negative
labels. This leads to the setting of missing labels investigated in this
paper, where only positive labels are affected by noise (they can go missing),
whereas negative labels remain unchanged (no spurious labels). For a formal
definition of the setting we refer the reader to \autoref{section:setting},
and for a more thorough discussion of prior works on missing labels and related
settings to \autoref{sec:related}.

Many machine learning methods are based on minimizing an expected loss over
the data distribution, typically by using the empirical risk as a statistical
estimator. Thus, a natural extension to the missing-labels setting is to
construct an unbiased estimator of the true risk given noisy data. In the XMC
context, such an approach has been introduced by \citet{Jain16}, who
constructed \emph{propensity-scored} versions for some common loss functions.
They achieve this under the assumption that the probabilities for each label to
not go missing (called its propensity) is known, and developed an empirical
model to estimate these propensities from data statistics. The model assumes
that propensities are identical for every data point (labels go missing
independently of features) and contains two dataset-dependent parameters that
have only been determined and made available for a few benchmark datasets.
Despite these shortcomings, the resulting propensity-scored metrics have found
widespread use in the XMC setting \citep{repo}.

However, many loss functions that are employed for training, such as binary
cross-entropy or the squared-hinge loss, do not fall within the scope of
\citet{Jain16}. Consequently, many methods that use propensity-scored
precision as an evaluation metric still perform training using a loss function
designed for clean-label training
\citep{Dahiya21b,guo_breaking_2019,attentionxml,babbar_data_2019}.

Based on the unbiased estimators given in
\citet{natarajan_cost-sensitive_2017} for the setting of class-conditional
noise, \citet{qaraei_convex_2021} provide unbiased versions for several common
loss functions. Similar to related learning settings
\citep{kiryo_positive-unlabeled_2017,chou_unbiased_2020}, they observed that
the unbiased estimates may be non-convex, non-lower-bounded, and lead to
severe overfitting due to high variance. This paper provides some additional
analysis in the form of a uniqueness result \autoref{thm:unique} that implies
that there are no other unbiased estimates with reduced variance, and a
generalization bound \autoref{thm:genbound} which suggests the bias-variance
trade-off observed in practice. A mitigation strategy is to interpret the loss
functions as convex surrogates of the 0-1 loss, and switch from unbiased
estimates of surrogates to convex surrogates of the unbiased estimate
\citep{qaraei_convex_2021,chou_unbiased_2020}.

Often, XMC problems are formulated as ranking tasks in which the goal is to
maximize either recall or precision within the top-k predictions. Instead of
optimizing this metrics directly, the task is typically reduced to a series of
binary or multiclass problems, with different reductions consistent for either
recall or precision \citet{menon_multilabel_2019}. The reductions consistent
for precision lead to loss functions that can be decomposed into a sum of
contributions from each label which makes them amenable to the methods of
\citet{natarajan_cost-sensitive_2017}. In contrast, the reductions consistent
for recall contain a normalization term that is the inverse of the total
number of true labels. This term is also necessary for calculating the recall
metric itself, demonstrating the need for unbiased estimates for true,
non-decomposable multilabel loss functions.

The unique, unbiased estimate for the generic multilabel case is provided by
\autoref{thm:efficient-multilabel}. This result can be seen as a special case
of \citet[Theorem 5]{van_rooyen_theory_2017}, which states that the unbiased
estimate can be calculated applying the inverse of the adjoint of the
label-corruption operator, a $2^l \times 2^l$ matrix for a problem with $l$
possible labels, to the vector of all $2^l$ possible loss values for a given
prediction. The solution that comes out of our direct computations requires
computation exponential only in the number of observed labels. In large-scale
multilabel problems, the number of relevant labels is typically much smaller
than that of possible labels \citep{Jain16}. If it grows logarithmically, then
our approach requires only $O(l)$ evaluations of the loss function.

We derive these results on the basis of modelling the observed labels as the
product of the true labels and an (unknown) mask variable. Note that this is
different from semi-supervised learning with labelled and unlabelled data, as
the mask is only used as a mathematical convenience, but no knowledge of the
actual mask values is assumed. The advantage of this formulation is that we
can choose a realization of the mask variable such that it is independent of
all other random variables, even though the modeled noise is
class-conditional.

In order to judge the severity of the variance and overfitting problems in
practice, we conducted two experiments. In the first, we conducted two
experiments. First, in a pure evaluation setting, we calculated the unbiased
recall@k for a varying fraction of missing labels, which shows that once this
fraction becomes too large, the variance of the estimate explodes and it
becomes unusable. As a consequence, this metric can only be calculated on
datasets with moderate propensities and many data points. The second
experiment serves as a demonstration for the change in bias-variance trade-off
as a result from switching to unbiased estimates. Here we trained a linear
model with varying $L_2$-regularization using the different versions of the
loss functions. We find that training with noisy labels generally shifts the
trade-off towards higher regularization, and that for the high-variance
unbiased estimates of non-decomposable losses, the original version of the
loss function works better than the unbiased one.

To summarize, the key contributions of this paper are
\begin{enumerate}
    \item The model for missing labels as a product of true labels and an
	independent mask variable (\autoref{def:masking}), which allows for convenient handling of
	expectations in proofs and derivations. We first demonstrate its usefulness
	in the binary case, where we derive new results on unqiueness (\autoref{thm:unique}) and variance
	of the unbiased estimators, and provide a generalization bound (\autoref{thm:genbound}) which is a corrected
	version of \citet[Lemma 8]{natarajan_cost-sensitive_2017}.
    \item The unbiased estimates for general multilabel functions (\autoref{thm:efficient-multilabel}). These can be applied 
    to the normalized loss reductions \citep{menon_multilabel_2019}, which are required for
    training that is consistent for recall@k. It turns out that even the unbiased estimation of recall@k
    becomes a very involved procedure (in constrast to precision@k) that requires summing over contributions from all subsets of the observed labels. 
    \item The investigation of the influence of missing labels and unbiased
    estimates on the bias-variance trade-off. We confirm the intuition given by the generalization bound that 
    training with the unbiased losses can lead to severe overfitting, requiring a re-tuning of regularization, and
    in some cases entirely negating the benefit of unbiased estimation. In situations where they are available,
    convex upper-bounds can be used to mitigate this problem. 
\end{enumerate}

\section{Notation and Setting}
\label{section:setting}

In this paper, random variables will be denoted by capital letters $X,
Y, \ldots$, whereas calligraphic letters denote sets and lower case
letters their elements, $x \in \mathcal{X}, \ldots$. Vectors will be denoted
by bold font, $\mathbf{y} \in \mathcal{Y}$, if we plan to make use of the fact
that they can be decomposed into components $y_1, \ldots, y_k$. 
The letters $f$, $g$, and $h$ are reserved for functions, $i$, $j$, $k$ denote integers. With $[k]$ we
denote the set $\set{1, \ldots, k}$, and $\allmaps{\mathcal{A}}{\mathcal{B}}$
is the set of all measurable mappings from $\mathcal{A}$ to $\mathcal{B}$.

There are two natural ways to express a multilabel data point for
$\numlabels$ possible labels: Either as vectors from $\set{0,1}^\numlabels$ or as subsets of $[\numlabels]$. We will mostly use the
former, and thus set $\labelspace = \set{0,1}^\numlabels$. In cases where the
subset notation is convenient, we use the symbol $\mathcal{I}(\vect{v})
\coloneqq \set{i \in [\numlabels]: v_i = 1}$ to denote conversion from subset
to vector representation, and $\indvect{\mathcal{I}(v)} = v$ for the reverse
operation.

Throughout this paper, we assume an abstract probability space $(\Omega,
\mathcal{F}, \prob)$. We further denote with $\dataspace$ the 
\emph{data space}, $\labelspace$ the \emph{label space} and
$\hat{\mathcal{Y}}$ the \emph{prediction space}. A dataset is defined through
the three random variables $\datapvar \in \dataspace$, $\vect{\obslabels} \in
\labelspace$, and $\vect{\truelabels} \in \labelspace$, that represent the
\emph{data}, \emph{observed label}, and \emph{ground truth label}. We will
generally mark quantities pertaining to the unobservable ground-truth with a
superscript star and call $(\datapvar, \vect{\truelabels})$ the \emph{clean data}.

For a given loss function $\defmap{f}{\labelspace \times
\predictspace}{\real}$, we denote the true and observed risks of a classifier $\defmap{h}{\dataspace}{\predictspace}$ as
\begin{align*}
    \trisk_f[h] \coloneqq \expect[f(\vect{\truelabels}, h(\datapvar))], \orisk_f[h] &\coloneqq \expect[f(\vect{\obslabels}, h(\datapvar))],
\end{align*}
and mark their empirical counterparts as $\terisk_f$ and $\oerisk_f$.

In this paper, we are interested in noisy labels where the noise is such that
labels can only go missing. This is described by the next two definitions,
where the first gives a phenomenological characterization of the setting,
whereas the second defines the mathematical model we use to describe the
setting.

\begin{definition}[Propensity]
\label{def:propensity}
The missing-labels setting we described informally in the introduction leads to the following
conditions on the $\numlabels$ random variables
    \begin{align}
        \cprob{\obslabels_j=1}{\truelabels_j=1, \datapvar}   &\eqqcolon \propensity_j   & & \text{labels may go missing,} \label{eq:def_prop} \\
        \cprob{\obslabels_j=1}{\truelabels_j=0, \datapvar}   &= 0                       & & \text{but no spurious labels.}\label{cond:nofp}
    \end{align}
    The value $\propensity_j \in (0, 1]$ is called the \emph{propensity} of the
    label $j$.
\end{definition}
In the model of \citet{Jain16}, for a dataset of $n$ points in which the label $j$ 
occurs $n_j$ times, and given dataset dependent constants $a, b$ and $c \coloneqq (\log n -1)(b+1)^a$ the propensity is determined by
\begin{equation}
    \propensity_j = (1 + c \exp(-a \log(n_j + b))^{-1}.
\end{equation}

\begin{definition}[Masking Model]
    \label{def:masking}
    The relation between $\vect{\truelabels}$ and $\vect{\obslabels}$ can be
    modelled by a set of \emph{mask} random variables $\vect{M} \in
    \set{0,1}^l$ such that $\vect{\obslabels} = \vect{M} \odot
    \vect{\truelabels}$. We require that the masks be independent of data and
    labels.
\end{definition}
This might seem like a restriction compared to \citet{Jain16}, as they
do not appear to make any independence assumptions. However, note that their
theorems require fixed predictions, which in our notation corresponds to
constant $\datapvar$, such that the independence of $\datapvar$ is implicitly
fulfilled. Further, they assume per-example knowledge of propensities, \ie
given a dataset $((x_1, \vect{y_1}), \ldots, (x_n, \vect{y_n}))$ they assume
knowledge of an entire propensity matrix $\propensity_{ij}$, $i \in [n], j \in
[l]$, \ie the propensity needs to be known for each instance $i$ and label $j$. 
Their empirical model for estimating propensities only produces
per-dataset values, resulting in the implicit assumption that
$\propensity_{ij} = p_{kj}$ $\forall i, k \in [n], j \in [l]$. 

To fulfill the conditions of our one-sided class-conditional noise model, we require 
\begin{align}
    \cprob{\obslabels_j=1}{\truelabels_j=1} &= \cprob{\truelabels_j=1, M_j=1}{\truelabels_j=1} = \cprob{M_i=1}{\truelabels_j=1} \stackrel{!}{=} \propensity_j\\
    \cprob{\obslabels_j=1}{\truelabels_j=0} &= \cprob{\truelabels_j=1, M_j=1}{\truelabels_j=0} = 0,
\end{align}
where we have left out the conditioning on $\datapvar$ for notational convenience.
Thus we can choose $\vect{M}$ such that
\begin{equation}
     \propensity_j = \cprob{M_j=1}{\truelabels_j=1} = \cprob{M_j=1}{\truelabels_j=0} = \prob[M_j=1] = \expect[M_j],
\end{equation} 
which means that $M_j$ is independent of the vector of true labels $\vect{\truelabels}$ and data point $\datapvar$.

\section{Unbiased Estimation for Binary Losses}
\label{sec:theory}


In this section, we first derive a general solution for the unbiased
estimation of binary losses and show its uniqueness. Furthermore, we prove a
generalization bound that demonstrates a bias-variance trade-off between the
unbiased and the original loss. Finally, we show how the binary results can
be applied to multilabel cases where the loss decomposes, which corresponds
to the \emph{one-vs-all} (OvA) and \emph{pick-all-labels}
(PAL) reductions \citep{menon_multilabel_2019}.

The technique used here relies heavily on two properties of the problem:
\begin{itemize}
  \item The label space is discrete, so we can rewrite the loss function 
  as a sum over contributions for each individual label setting.
  \item The independence properties of $\vect{M}$, which allows replacing 
  expectations over products containing $M_i$ by a multiplication with $\propensity_i$. 
\end{itemize}

\subsection{Derivation of Unbiased Binary Losses}
As the considerations here are limited to the binary case, we will drop the
vector notation and corresponding subscripts for the labels. Note that binary
is to be understood in the sense of detecting the presence of absence of some
label ("is there a dog in this picture?"), not as a decision between two
classes ("does this picture show a dog or a cat?").

First, we define a \emph{propensity-scoring} operator that maps
a function to a surrogate function that can be used to compensate 
for missing labels, and prove that this does lead to unbiased estimates.

\begin{definition}[PS Operator] 
    \label{def:ps-operator}
    Let $\mathcal{Z}$ be an arbitrary set$, \mathcal{V}$ a vector 
    space, and $\defmap{f^*}{\set{0,1} \times \mathcal{Z}}{\mathcal{V}}$ be 
    some function. Since the first argument can only take the two different 
    values 0 and 1, we can decompose
    \begin{gather}
        f^*(y, z) \eqqcolon y \opl(z) + (1-y) \oml(z).
    \intertext{Then, for any $p \in (0, 1]$, we call $\defmap{f}{\{0,1\} \times \mathcal{Z}}{\mathcal{V}}$ defined as}
        f(y, z) \coloneqq y\frac{\opl(z) + (p-1) \oml(z)}{p} + (1-y) \oml(z)
    \end{gather}
    a \emph{propensity-scored} version of $f^*$, and the 
    mapping $\pswt: f^* \mapsto f$ the \emph{propensity-scoring (PS) operator}
    for propensity $p$.
\end{definition}
Note that, in general, convexity and non-negativeness of $f^*$ need not result
in convexity and non-negativeness of $f$, \cf section \ref{ssec:bce}. The PS
operator results in unbiased estimates, as shown below:

\begin{theorem}[Unbiased Estimates with Missing Binary Labels]
    \label{thm:single_label}
    Assuming the masking model, let $k \in \natnum$ and $\defmap{f^*}{\{0,1\} \times \mathcal{X}}{\real^k}$,
    then $f \coloneqq \pswt(f^*)$ allows to calculate an unbiased estimate of $f^*$ by
    \begin{equation}
        \expect[f^*(Y^*, \datapvar)] = \expect[f(Y, \datapvar)].
    \end{equation}
\end{theorem}
\begin{proof}
    Using the linearity of the expectation and independence of $M$, we can calculate
    \begin{align*}
         \MoveEqLeft
        \expect[f(\obslabels, \datapvar)] = \expect[M \truelabels \cpl(\datapvar) + (1-M \truelabels) \cml(\datapvar) )] \\
        &= \expect[M] \expect[\truelabels (\cpl(X) - \oml(X))] + \expect[\oml(X)] \label{eq:thm-single-label-indep} \\
        &= 
            p \expect[\truelabels (p^{-1} \left(\opl(X) + (p-1) \oml(X) \right)- \oml(X)] + \expect[\oml(X)]
         \\
        &= \expect[\truelabels \left(\opl(X) - \oml(X) \right)] + \expect[\oml(X)] \\
        &= \expect[f^*(\truelabels, X)]. & \hspace{2em} \qedhere
    \end{align*}
\end{proof}

This theorem corresponds to \citet[Lemma 7]{natarajan_cost-sensitive_2017}. In practice,
the argument to the loss function is a prediction given by some classifier $\phi$, so
what we need is the following
\begin{restatable}[Binary Loss Function's Gradient]{corollary}{thmbingrad}
    \label{col:binary_grad}
    Let $\defmap{\phi}{\mathcal{X} \times \real^k}{\real}$  be a binary classifier
    with $k \in \natnum$ parameters, and $\defmap{\ell^*}{\set{0, 1} \times
    \real}{\real}$ be a loss function, $\ell = \pswt(\ell^*)$. Assume that the mask $M$ is independent of $(\datapvar, \truelabels, \vect{W})$. Then 
    \begin{align*}
        \expect[\ell^*(\truelabels, \phi(\datapvar, \vect{W}))] &= \expect[\ell(\obslabels, \phi(\datapvar, \vect{W}))] \\
        \expect[\nabla_{\vect{W}} \ell^*(Y^*, \phi(\datapvar, \vect{W}))] &= \expect[\nabla_{\vect{W}} \ell(Y, \phi(\datapvar, \vect{W}))].
    \end{align*} 
\end{restatable}
\begin{proof}
    The independence of $\datapvar, \vect{W}$ and $M$ implies the independence
    of $\phi(\datapvar; \vect{W})$ and $M$. For the second equation, apply
    \autoref{thm:single_label} with $\dataspace^\prime = \dataspace \times
    \real^k$ to the gradient of the loss function and use linearity of the
    $\pswt$ operator, expectation, and gradient. For details, see appendix
    \ref{proof:corbingrad}.
\end{proof}

\subsection{Examples}
\paragraph{Squared Error}
\label{ssec:SE}
\begin{wrapfigure}[19]{r}{0.4\linewidth}
    \begin{tikzpicture}
    \begin{axis}[
        width=\linewidth,
        xlabel={$\hat{y}$},
        ylabel={loss},
        ylabel shift={-1em},
        xmin=0, xmax=2.0, ymax=3.0, ymin=-1.0,
        very thick,
        domain=0:2.0,
        samples=50
    ]
    \addplot[blue,dashed] {1/0.66*(1-2*x) + x*x};
    \addplot[blue] {0/0.66*(1-2*x) + x*x};
    
    \end{axis}
    \end{tikzpicture}
    \vspace{-1em}
    \caption{
    Propensity-scored squared error loss. The dashed lines 
    denote the loss for $y=1$, the solid ones for $y=0$. The propensity was chosen as $p=0.66$.
    }
    \label{fig:overfitting-se}
\end{wrapfigure}
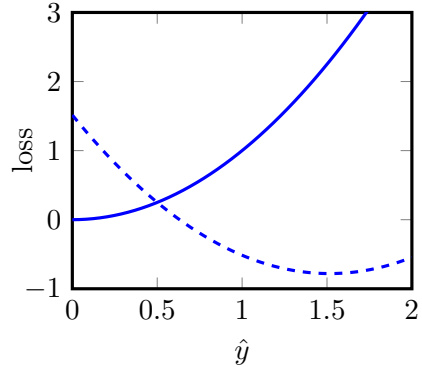

For the squared error $f^*_{\mathrm{SE}}(y, \hat{y}) = (y - \hat{y})^2$ the unbiased estimate is
\begin{equation}
    f_{\mathrm{SE}}(y, \hat{y}) = \frac{y}{p} \left( 1 - 2 \hat{y}\right) + \hat{y}^2,
\end{equation}
where we used $y^2=y$ because $y \in \set{0, 1}$. In case the label is present
($y=1$), this is minimized for $\hat{y} = \frac{y}{p}$. The minimum is outside
the interval $[0, 1]$, as shown in \autoref{fig:overfitting-se}.

\paragraph{Binary Cross-Entropy}
\label{ssec:bce}
The BCE loss is given by
\begin{equation}
    f^*_{\mathrm{BCE}}(y, \hat{y}) = -y\log(\hat{y}) - (1-y)\log(1 - \hat{y}).
\end{equation}
Since this is linear in $y$, we can directly write down the unbiased estimate as
\begin{equation}
    f_{\mathrm{BCE}}(y, \hat{y}) = -\frac{y}{p} \log(\hat{y}) - \left(1-\frac{y}{p} \right)\log(1 - \hat{y}). \label{eq:ps-bce}
\end{equation}
Note that, since the BCE loss is not upper-bounded, the unbiased BCE loss is
not lower-bounded, which may be bad from an optimization perspective.

\subsection{Properties of the Unbiased Estimator}
Even though the unbiased estimate guarantees to give correct results in the
case of unlimited data, it is not directly clear how helpful it is in the
finite data regime. In the related problems of PU-learning and learning from
complementary labels, unbiased estimators are known but they have been found
to be problematic in practice
\citep{kiryo_positive-unlabeled_2017,chou_unbiased_2020}.

A first problem is that even if the original loss has been chosen as a convex
function, the unbiased estimator may not be convex. A sufficient condition for
convexity is given in \citet[Lemma 10]{natarajan_cost-sensitive_2017}. In
fact, if the original loss is not upper-bounded (\eg hinge-loss, logistics
loss, squared loss), then the unbiased estimate may not even be lower-bounded,
thus rendering the optimization-problem ill-defined. A way of addressing this
problem is to forgo the unbiasedness and switch to convex surrogate losses as
discussed in \autoref{sec:upper-bounds}.

\paragraph{Variance}
A second potential problem with unbiased estimators is their variance.  In the
regime of $\propensity \rightarrow 0$, we can show that in the binary case the
variance grows with $\propensity^{-1}$ compared to the evaluation on clean
data. For the noiseless case, the variance is given by
\begin{align}
    \variance[f^*(\hat{y}, \truelabels)] &= \variance[\truelabels\opl(\hat{y})+ (1-\truelabels) \oml(\hat{y})] \nonumber\\
    &= \variance[\truelabels]\left( \opl(\hat{y}) - \oml(\hat{y}) \right)^2 \\
\shortintertext{whereas the noisy estimator has a variance of}
    \variance[f(\hat{y}, \obslabels)] &= \variance[\obslabels\frac{\opl(\hat{y}) + (\propensity-1) \oml(\hat{y})}{\propensity} + (1-\obslabels) \oml(\hat{y})] \nonumber \\
    &= \variance[\obslabels] \left( \frac{\opl(\hat{y}) + (\propensity-1) \oml(\hat{y})}{\propensity} \right)^2 + \variance[\obslabels] \oml(\hat{y})^2 \nonumber \\
    &= \variance[\obslabels] \frac{\left( \opl(\hat{y}) + (\propensity-1) \oml(\hat{y}) \right)^2 +  \propensity^2\oml(\hat{y})^2}{\propensity^2}.\\ 
\shortintertext{For propensities much smaller than 1, this can be approximated by (recall that $q \coloneqq \expect[\obslabels]$)}
     &\approx \variance[\obslabels] \frac{ \left( \opl(\hat{y}) - \oml(\hat{y}) \right)^2 }{\propensity^2}
    = q(1-q) \frac{ \left( \opl(\hat{y}) - \oml(\hat{y}) \right)^2 }{\propensity^2}.
\end{align}
Setting $q^* \coloneqq \expect[\truelabels] = q/p$, and using $1-q = 1 - p q^* \approx 1$ we get
\begin{equation}
    \variance[f(\hat{y}, \obslabels)] \approx \frac{q(1-q)}{\propensity^2} \frac{ \variance[f^*(\hat{y}, \truelabels)]}{q^* (1-q^*)} 
    \approx \frac{1}{\propensity(1-q^*)} \variance[f^*(\hat{y}, \truelabels)],
\end{equation}
which means that the variance increases linearly with inverse propensity.

\paragraph{Generalization}
The preceding argument indicates that there might be a bias-variance trade-off
between using the unbiased loss that may overfit more strongly on the observed
noise, and using the original loss function which gives wrong results even if
$n \rightarrow \infty$. A first step toward understanding this is to determine
upper bounds on the generalization errors (proof in appendix \ref{proof:thmgenbound}):
\begin{restatable}[Generalization bounds]{theorem}{thmgenbound}
    \label{thm:genbound}
    Let $\mathcal{H}$ be a function class with Rademacher complexity
    $\rademacher_n(\mathcal{H})$. Let $\defmap{f^*}{\labelspace \times
    \predictspace}{[a, b]}$ for $a < b \in \real$ be a loss function that is
    $\rho-$Lipschitz continuous in its second argument. Let $f \coloneqq \pswt f^*$ and denote
    \begin{gather}
        r^{\star} \coloneqq \inf_{h \in \mathcal{H}} \trisk_{f^*}[h], \quad
        \hat{h} \coloneqq \argmin_{h \in \mathcal{H}} \oerisk_{f}[h], \quad
        \tilde{h} \coloneqq \argmin_{h \in \mathcal{H}} \oerisk_{f^*}[h].
    \shortintertext{as well as}
    \begin{aligned}
        c &\coloneqq \rho \rademacher_n(\mathcal{H}) + (b-a) \sqrt{\frac{\log(2/\delta)}{2n}}\\
        m &\coloneqq \sup_{x \in \dataspace}(\abs{f^*(1, x) - f^*(0, x)}).
    \end{aligned}
    \end{gather}
    For a given sample of $n$ points, it holds with probability at least $1-\delta$
    \begin{align}
        \hat{r} &\coloneqq \trisk_{f^*}[\hat{h}] &{}\leq{}& r^{\star}&& &{}+{}& 2 \frac{2-p}{p} c \\ 
        \tilde{r} &\coloneqq \trisk_{f^*}[\tilde{h}]  &{}\leq{}& r^{\star} + &q \frac{1 - p}{p} m& &{}+{}& 2 c,
    \end{align}
    where $q \coloneqq \expect[\obslabels] \leq 1$.
\end{restatable}

Given a hypotheses class $\mathcal{H}$, the (expected) \emph{Rademacher complexity} for $n$ sample points is defined by
\begin{equation}
    \rademacher_n(\mathcal{H}) = \expect[\sup_{h \in  \mathcal{H}} \sum_{i=1}^n \sigma_i h(X_i)],
\end{equation}
where $\sigma_i \in \set{-1, 1}$ are independent Rademacher variables (cf. \citet[Ch. 26]{shalev2014understanding}).

The bound on the original loss function has, in addition to the Bayes and
approximation errors error $r^\star$, a second bias term $q \frac{1 - p}{p} m$
that does not decrease as $n$ increases. On the other hand, the unbiased
estimation introduces a factor of $\frac{2-p}{p}$ in front of the
sample-size-dependent term. This indicates that, if only few training samples
are available, the ERM of the unbiased loss might result in a worse
classifier.

\paragraph{Uniqueness}
The results above raise the question whether there might be other unbiased
estimators with reduced variance and better generalization performance. For
example, the conditional expectation $\expect[f^*(Y^*, X) | Y]$ also gives an
unbiased estimate with lower variance, but cannot be calculated without
knowledge of the conditional probabilities $\cprob{Y}{X}$. The following
theorem states that $\pswt$ is essentially unique if the marginal probability
of the label is known beforehand, and completely unique otherwise, and thus we
cannot reduce the variance.

\begin{restatable}[Uniqueness]{theorem}{thmpsuniqueness}
    \label{thm:unique}
     Let $|\dataspace| \geq 2$, $p \in (0, 1]$, $q \in (0, p]$, and
     $\mathcal{F} = \allmaps{\{0,1\} \times \mathcal{X}}{\real}$ be a set of
     functions. Let $\defmap{\mathfrak{P}}{\mathcal{F}}{\mathcal{F}}$ be an
     operator that maps a function to an unbiased estimate in the missing
     labels setting, such that for all $X, Y, Y^*$ that fulfill the masking
     model with propensity $p$ and marginal $q$, it holds
     \begin{equation}
         \forall f^* \in \mathcal{F}: \expect[f^*(Y^*, X)] = \expect[\mathfrak{P}(f^*)(Y, X)]. \label{eq:unbiased-operator}
     \end{equation}
     Then, $\mathfrak{P}$ is related to the propensity scoring operator $\pswt$ by
     \begin{equation}
         \mathfrak{P}(f^*)(y, x) = \pswt(f^*)(y, x) + (y - q) \gamma
     \end{equation}
     for some $\gamma \in \real$. Since this does not affect the dependency on
     $x$, from a learning perspective, the propensity-scoring operator $\pswt$
     is essentially unique.
\end{restatable}
\begin{proof}
The sufficiency follows from \autoref{thm:single_label}, the necessity can 
be shown by choosing distributions of $X, Y^*$ in which $X$ is concentrated on two points. See \ref{proof:uniqueness}.
\end{proof}

\subsection{Multilabel Losses that Decompose to Binary Losses}
By linearity of the expectation, we can trivially extend the results
above to multilabel loss functions of the following form:
\begin{definition}[Decomposable Loss Function]
    A multilabel loss function $\defmap{f^*}{\labelspace \times \real^\numlabels}{\real}$
    is called \emph{decomposable} if it can be written as
    \begin{equation}
        f^*(\vect{y}, \vect{\hat{y}}) = \sum_{i=1}^{\numlabels} f^*_{i}(y_i, \vect{\hat{y}}),
    \end{equation}
    with $\defmap{f^*_{i}}{\set{0,1} \times \real^\numlabels}{\real}$. Other multilabel loss
    functions are called \emph{non-decomposable}.
\end{definition}

For decomposable loss functions, we can formulate the following
\begin{corollary}[Unbiased Estimate for Decomposable Loss] Let $f^*$ be a
 decomposable multilabel loss function with constituent functions $f^*_i$ and
 corresponding propensity scored functions $f_i \coloneqq \pswt[i](f^*_i)$.
 Then the function 
 \begin{equation}
     f(\vect{y}, \vect{\hat{y}}) \coloneqq \sum_{i=1}^{\numlabels} f_{i}(y_i, \vect{\hat{y}}),
 \end{equation} results in an unbiased estimate of the true loss when applied
  to noisy labels generated according to the masking model.
\end{corollary}
\begin{proof}
Linearity of the expectation.
\end{proof}

As corollaries we can write down the unbiased estimates for two important
special cases of multilabel reductions: the OvA and the PAL reduction
\citep{menon_multilabel_2019}. In the OvA reduction, both the true/observed
labels and the prediction are decomposed into the binary setting, such that
the constituent functions are generated from binary losses
$\defmap{g^*_{i}}{\set{0,1} \times \real} {\real}$ via $f^*_i(y_i,
\vect{\hat{y}}) = g^*_i(y_i, \hat{y}_i)$. In the typical case, the binary
losses are identical for all labels, but we keep the more general formulation
to allow for different weightings or margin parameters depending on label
frequency as a mitigation for label imbalance.
\begin{corollary}[One-vs-All]
    For the one-vs-all reduction with binary losses $g^*_i$, an unbiased estimate
    can be calculated by
    \begin{equation}
        f(\vect{y}, \vect{\hat{y}}) \coloneqq \sum_{i=1}^{\numlabels} \pswt[i]\left(g^*_i\right)(y_i, \hat{y}_i).
    \end{equation}
\end{corollary}
The pick-all-labels reduction also imposes a restriction on the structure of
$f^*_i$, as the PAL reduction only gets nonzero contributions when the label
is present. 
\begin{corollary}[Pick-all-Labels]
    For a PAL loss function
    \begin{equation}
        f^*(\vect{y}, \vect{\hat{y}}) = \sum_{i=1}^{\numlabels} y_i f^*_{i}(\vect{\hat{y}}),
    \end{equation}
    the unbiased estimate is given by the function
    \begin{equation}
        f(\vect{y}, \vect{\hat{y}}) \coloneqq \sum_{i=1}^{\numlabels} \frac{y_i}{\propensity_i} f^*_{i}(\vect{\hat{y}}).
    \end{equation}
\end{corollary}

This means that in the PAL setting, the unbiased estimate is just a weighted
sum of the original constituent functions. In particular, if these are
convex (bounded), then the unbiased estimate will also be convex (bounded).

\section{Unbiased Estimation for Multilabel Losses}
\label{sec:multi-label}
In this section we consider the general multilabel case in which we cannot
decompose the loss into per-label contributions. We first derive a generic
expression and properties, then apply it to two special cases that serve as 
building blocks for many concrete multilabel losses. At the end of the section
we derive expressions for normalized PAL and OvA decompositions as well as 
recall and pairwise losses.

\subsection{Generic Multilabel Case}
For the general multilabel case, we can state
\begin{restatable}[Multilabel Loss]{theorem}{thmmultilabel}
    \label{thm:efficient-multilabel}
    Let $l \in \natnum$ and \defmap{f^*}{\set{0,1}^l \times
    \mathcal{X}}{\real} be some mapping. An unbiased estimate of this function
    can be calculated using the propensity-scored expression $\mathfrak{P}_{\mathbf{p}}(f^*) = f$ given by
    \begin{equation}
      (\vect{y}, x) \mapsto \left( \prod_{i \in \mathcal{I}(\vect{y})} \frac{1}{p_i} \right) \cdot \sum_{\mathclap{\mathcal{J} \subset \mathcal{I}(\vect{y})}} f^*(\indvect{\mathcal{J}}, x) \prod_{j \in \mathcal{I}(\vect{y}) \setminus \mathcal{J} }\left(p_j-1\right),
        \label{eq:ps-multilabel-eff}
    \end{equation}
    where $\indvect{\mathcal{J}} \in \set{0,1}^l$ has entry one for all indices $j \in \mathcal{J}$. 
\end{restatable}
\begin{proof}
    Decompose the function $f^*$ into contributions for
    each possible value of $\vect{Y}$ as $f^*(\vect{Y}, X) = \sum_{\vect{y}}
    \ind[\vect{Y} = \vect{y}] f^*(\vect{y}, X)$ and write the indicator as
    products of $Y_i$ and $1-Y_i$. Then use the linearity of the expected value as
    well as the independence of $\vect{M}$. For details, see supplementary \ref{proof:multilabel}.
\end{proof}

If we apply \autoref{thm:efficient-multilabel} to the binary case, we get
\begin{align}
    f(0, x) &= f^*(0, x)  \qquad (J=\emptyset)\\
    f(1, x) &= \frac{p - 1}{p} \cdot \left( f^*(0, x) + f^*(1, x) \left(p-1\right)^{-1} \right)
     \nonumber \\
     &=p^{-1} \cdot \left( f^*(0, x)(p-1) + f^*(1, x) \right),
\end{align}
thus recovering \autoref{thm:single_label}.

Unfortunately, the sum-over-all-subsets structure makes computation of
this estimate expensive. However, in
cases with large label spaces, the actual label vectors typically become quite sparse,
which mitigates the impact to some degree, see \autoref{rem:cc}. This is still much
more efficient than the general solution in \citet[Theorem
5]{van_rooyen_theory_2017}, which is given in terms of transitions between
observations. In our case, each of the $2^l$ possible combination of relevant
labels corresponds to one distinct observation. Their unbiased estimate is
\begin{equation}
    \vect{f}(\hat{y}) = R^* \vect{f^*}(\hat{y}),
\end{equation}
where $R^*$ denotes the adjoint of the inverse of the corruption matrix,
and $\vect{f}$ and $\vect{f^*}$ denote vectors that contain the loss for every possible observation.
This means that naively, one would need to evaluate all $2^l$ possible values of $f^*(\cdot, \vect{y})$
and combine them using a (sparse) $2^l \times 2^l$ matrix.

As in the binary case, we can state a uniqueness theorem : 
\begin{restatable}[Multilabel Uniqueness]{theorem}{thmmultilabelunique}
\label{thm:multi-unique}
Let $\mathfrak{P}_{\vect{p}}$ denote the propensity scoring operator from \autoref{thm:efficient-multilabel}, 
and $\mathfrak{P}$ another operator such that
\begin{equation}
    \expect[f^*(\vect{Y}^*, \datapvar)] = \expect[\mathfrak{P}(f^*)(\vect{Y}, \datapvar)] \label{eq:multi-unique-condition}
\end{equation}
for all $f^*$, and all distributions of $\vect{Y}^*$ and $\datapvar$. Then
$\mathfrak{P}(f^*) = \mathfrak{P}_{\vect{p}}(f^*)$.
\end{restatable}
\begin{proof}
    Since $\mathfrak{P}$ needs to work for all possible distributions of
    $\datapvar$ and $\vect{Y}^*$, it needs to work in particular also for
    $\prob[X=x, \vect{Y^*}=\vect{y}] = 1$. Since $\vect{\obslabels}$ can take
    only finitely many states, we can decompose $\mathfrak{P}(f^*)$ into a sum
    over these states. The claim can then be shown by induction over the
    number of nonzero elements in $\vect{y}$, which always introduces only a
    single new summand in the decomposition. Details can be found in the
    supplementary \ref{proof:multi-uniqueness}.
\end{proof}

\subsection{Normalized Multilabel Reductions}
We first introduce a general form of a \emph{normalized reduction}, which
allows unified treatment of PAL-N and OvA-N reductions, as well as (per
instance) recall as a concrete loss functions. The normalized reductions
are derived from the basic reductions by replacing the label $\truelabels$
with a rescaled variable $\tilde{\truelabels_i} = \frac{\truelabels_i}{\sum_{j=1}^{\numlabels} \truelabels_j}$.
\begin{definition}
    A multilabel loss $f^*$ is considered to be a \emph{normalized reduction}
    if there exist functions $g_i$ and $h$ such that
    \begin{equation}
        f^*(\vect{y}, \hat{\vect{y}}) = h(\hat{\vect{y}}) + \sum_{i=1}^{\numlabels} \frac{y_i}{\sum_{j=1}^{\numlabels} y_j} g_i(\hat{\vect{y}}).
    \end{equation}
\end{definition}
By defining a set of random variables $T^*_i$ as 
\begin{equation}
    T^*_i \coloneqq \frac{\truelabels_i}{\sum_{j=1}^{\numlabels}\truelabels_j}
\end{equation}
we can write the value of such a loss as $f^*(\truelabels, \hat{\vect{y}}) =
\sum_{i=1}^{\numlabels} T^*_i g_i(\hat{\vect{y}}) + h(\hat{\vect{y}})$. Thus,
for calculating the unbiased estimates it suffices to apply the generic
formula of \autoref{thm:efficient-multilabel} to $T^*_i$ and use the
linearity. We get
\begin{equation}
    T_i = \left( \prod_{j \in \mathcal{I}(\vect{y})} \frac{1}{p_j} \right) \cdot \sum_{\mathclap{\mathcal{J} \subset \mathcal{I}(\vect{y})}} \frac{\ind[i \in \mathcal{J}]}{|\mathcal{J}|} \prod_{k \in \mathcal{I}(\vect{y}) \setminus \mathcal{J}}\left(p_k-1\right).
\end{equation}
Therefore, the corresponding unbiased loss function is
\begin{equation}
     f(\vect{y}, \hat{\vect{y}}) = h(\hat{\vect{y}}) + \left( \prod_{j \in \mathcal{I}(\vect{y})} \frac{1}{p_j} \right) \sum_{\mathclap{\mathcal{J} \subset \mathcal{I}(\vect{y})}} \;\,  \frac{ \sum_{j \in \mathcal{J}} g_j(\hat{\vect{y}}) }{|\mathcal{J}|} \prod_{\mathclap{k \in \mathcal{I}(\vect{y}) \setminus \mathcal{J}}}\left(p_k-1\right).
     \label{eq:ubloss-normalized-nondecomposable}
\end{equation}

\paragraph{One-vs-All-Normalized}
Let $g_{\mathrm{BC}}$ be a binary loss function.\footnote{Generalization to
label-dependent loss functions is straightforward.} Then the OvA-N reduction is
defined as
\begin{equation}
    f^*_{\mathrm{OvA-N}}(\vect{y}, \hat{\vect{y}}) \coloneqq \sum_{i=1}^{\numlabels} \frac{y_i}{\sum_{j=1}^{\numlabels} y_j} g_{\mathrm{BC}}(1, \hat{y}_i) + \left( 1 - \frac{y_i}{\sum_{j=1}^{\numlabels} y_j} \right) g_{\mathrm{BC}}(0, \hat{y}_i).  \label{eq:ova-n-unbiased}
\end{equation}
By comparing terms, we see that this corresponds to $g_i(\hat{\vect{y}}) = g_{\mathrm{BC}}(1, \hat{y}_i) - g_{\mathrm{BC}}(0, \hat{y}_i)$
and $h(\hat{\vect{y}}) = \sum_{i=1}^{\numlabels} g_{\mathrm{BC}}(0, \hat{y}_i)$.

\paragraph{Pick-All-Labels-Normalized}
For a given multiclass loss $g_{\mathrm{MC}}$, the PAL-N reduction is given by
\begin{equation}
    f^*(\vect{y}, \hat{\vect{y}}) \coloneqq \sum_{i=1}^{\numlabels} \frac{y_i}{\sum_{j=1}^{\numlabels} y_j} g_{\mathrm{MC}}(i, \hat{\vect{y}}),
\end{equation}
which immediately gives the correspondence $g_i(\hat{\vect{y}}) =
g_{\mathrm{MC}}(i, \hat{\vect{y}})$ and $h \equiv 0$.

\paragraph{Per-Example Recall}

In a multilabel setting, we can define the recall (for a single example and
for all labels, as opposed to for a single label over the entire dataset) as
the fraction of relevant labels that have been predicted \citep{8036272}. Thus
for $\vect{\hat{y}} \in \set{0,1}^l$
\begin{equation}
    \operatorname{Rec}^*(\vect{y}, \vect{\hat{y}}) \coloneqq |\mathcal{I}(\vect{y})|^{-1} \sum_{\mathclap{i \in \mathcal{I}(\vect{y})}} \hat{y}_i. \label{eq:def:recall}
\end{equation}
This corresponds to the normalized PAL reduction with $g_{\mathrm{MC}}(i,
\hat{\vect{y}}) = \hat{y}_i$.

Note that $\vect{\hat{y}}$ usually is also a sparse vector, \eg when
calculating recall@k, it has exactly $k$ nonzero entries. This may allow for a
slightly more efficient calculation. Setting $\mathcal{S} \coloneqq
\mathcal{I}(\vect{y}) \cap \mathcal{I}(\vect{\hat{y}})$ and $\mathcal{T} =
\mathcal{I}(\vect{y}) \setminus \mathcal{S}$ and
rearranging terms (see appendix \ref{appendix:per-example-recall}), the
propensity-scored recall is given by
\begin{equation}
    \operatorname{Rec}(\vect{y}, \vect{\hat{y}}) = \prod_{i \in \mathcal{I}(\vect{y})} \frac{p_i - 1}{p_i} \cdot \sum_{s=1}^{|\mathcal{S}|} \sum_{\substack{\mathcal{K} \subset \mathcal{S} \\|\mathcal{K}| = s}} \prod_{k \in \mathcal{K}} \frac{1}{p_k - 1} \cdot s \left(\sum_{{\mathcal{J} \subset \mathcal{T}}} \frac{\prod_{j \in \mathcal{J}} \left( p_j - 1\right)^{-1} }{|\mathcal{J}| + s} \right) .
    \label{eq:ps-rec}
\end{equation}

One caveat of the unbiased estimate is that even though its expectation is
confined to the interval $[0, 1]$, the value for individual instances may be
very far away from that interval. As an illustration, consider this
\paragraph{Example}
Assume that $\numlabels=3$ and with probability 1 we have $\truelabels_0=\truelabels_1=1-\truelabels_2 = 1$. Let the prediction be $\predlabels_0=1, \predlabels_1=\predlabels_2=0$. For any observation
$\obslabels_0 = 0$, $\mathcal{S}$ is empty and we get $0$ recall. For $\obslabels_0 = 1$ and $\obslabels_1 = 0$, we have $\mathcal{S}=\set{0}, \mathcal{T} = \emptyset$
so that
\begin{equation}
    \operatorname{Rec} = \frac{p_0-1}{p_0} \cdot 1 \cdot \frac{1}{p_0 - 1} = \frac{1}{p_0}.
\end{equation}
Finally, if both $\obslabels_0 \obslabels_1 = 1$, then $\mathcal{S}=\set{0}, \mathcal{T} = \set{1}$ and 
\begin{equation}
\operatorname{Rec} = \frac{p_0-1}{p_0} \frac{p_1-1}{p_1} \left( 1 + \frac{(p_1-1)^{-1}}{2} \right) \cdot \frac{1}{p_0 - 1} 
=\frac{2(p_1-1) + 1}{2 p_0 p_1}.
\end{equation} 
In \autoref{tab:psrecall-example} the different states are listed explicitly for a propensity of $1/3$. The expectation becomes
\begin{equation}
    \expect[\operatorname{Recall}] = \frac{2}{9} \cdot 3 + \frac{1}{9} \cdot \frac{-3}{2} = \frac{2}{3} - \frac{1}{6} = \frac{1}{2},
\end{equation}
which is the true recall. However, note the strong increase in variance, such
that the standard deviation is larger than the interval in which we know the
true value to reside.

\begin{table}
    \centering
    \caption{Probability, vanilla recall, and propensity scored recall for different observed labels, given
    that the prediction is label 0 and the ground truth is labels 0 and 1.}
    \label{tab:psrecall-example}
    \pgfplotstabletypeset[BasicTableStyle,
        columns/Y0/.style={column name={$\obslabels_0$}},
        columns/Y1/.style={column name={$\obslabels_1$}},
        columns/Y2/.style={column name={$\obslabels_2$}},
        columns/P/.style={column name={$\prob$}, string type},
        columns/Rec/.style={column name={R@1}},
        columns/PsRec/.style={column name={PsR@1}, string type},
    ]{
        Y0 Y1 Y2    P         Rec     PsRec
        0 0 0       {$4/9$}   0         $0$
        0 1 0       {$2/9$}   0         $0$
        1 0 0       {$2/9$}   1         $3$
        1 1 0       {$1/9$}   0.5       $-3/2$
    }
  
\end{table}

\section{Upper-Bounds}
\label{sec:upper-bounds}
\subsection{Convex Upper Bounds for OvA Losses}
Many binary losses used in machine learning are convex surrogates of the
0-1-loss, for example the logistic loss, (squared) hinge loss, and squared
error. Thus, one way to arrive at convex loss functions adapted to missing
labels is to switch the order of operations: Instead of calculating unbiased
estimates of convex surrogates, we calculate the unbiased estimate of the
0-1-loss and take a convex surrogate of the resulting expression.

To that end, first consider the binary case: Let
$\defmap{f^*}{\set{0,1}\times\real}{\real}$ be a function that is convex in
its second argument and forms an upper-bound on the 0-1 loss. The unbiased
estimator for the 0-1 loss for positive label is given by
\begin{align}
    \pswt(f_{01})(1, \hat{y}) &=  \propensity^{-1} \left( \ind[\hat{y} \leq 0] + (\propensity - 1)  \ind[\hat{y} > 0] \right) \nonumber \\
    &= \propensity^{-1} \left( \ind[\hat{y} \leq 0] + (\propensity - 1) (1- \ind[\hat{y} \leq 0]) \right)  \nonumber \\
    &= \propensity^{-1} \left( (2 - \propensity) \ind[\hat{y} \leq 0] + \propensity - 1) \right)  \nonumber \\
    &= \left(2/\propensity - 1 \right) \ind[\hat{y} \leq 0] + \text{const.}
\end{align}
From an empirical-risk minimization perspective, the constant does not affect
the outcome, so it can be ignored. Therefore, a convex upper-bound to
this unbiased 0-1 loss (with the constant removed) is given by
\begin{equation}
    f(y, \hat{y}) = y \left( \frac{2}{\propensity} - 1 \right) f^*(1, \hat{y}) + (1-y) f^*(0, \hat{y}).  \label{eq:binary-upper-bound}
\end{equation}

\subsection{Upper Bounds for Normalized Multilabel Reductions}
We have formulated the normalized multilabel reductions in terms of the
variable
\begin{equation}
    T^*_i = \frac{\truelabels_i}{\sum_{j=1}^{\numlabels}\truelabels_j} = \frac{\truelabels_i}{1 + \sum_{j \neq i}\truelabels_j}.
\end{equation}
A naive attempt of correcting for the noisy labels is to replace $\truelabels$ by $\obslabels/\propensity$, given by
\begin{equation}
    \tilde{T}_i \coloneqq \frac{\obslabels_i / \propensity_i }{1 + \sum_{j \neq i} \obslabels_j / \propensity_j}, \label{eq:upper-bound-for-T}
\end{equation}
which is not unbiased. We can show, however, that this expression fulfills
\begin{align}
    \expect[\tilde{T}_i] &= \expect[ \frac{M_i \truelabels_i / \propensity_i }{1 + \sum_{j \neq i} M_j \truelabels_j / \propensity_j}] \\
    &= \expect[M_i / \propensity_i] \expect[ \frac{\truelabels_i}{1 + \sum_{j \neq i} M_j \truelabels_j / \propensity_j}] \tag{independence}\\
    &= \expect[ \expect[\frac{\truelabels_i}{1 + \sum_{j \neq i} M_j \truelabels_j / \propensity_j} \Bigm| \vect{\truelabels}]] \tag{tower}\\
    &= \expect[\truelabels_i \expect[\frac{1}{1 + \sum_{j \neq i} M_j \truelabels_j / \propensity_j} \Bigm| \vect{\truelabels}]] \tag{measurable factor} \\
    \shortintertext{Now we can apply Jensen's inequality for the second factor}
    \expect[\tilde{T}_i] &\geq \expect[ \frac{\truelabels_i}{1 + \expect[\sum_{j \neq i} M_j \truelabels_j / \propensity_j \mid \vect{\truelabels}]}] \tag{convexity}\\
    &= \expect[ \frac{\truelabels_i}{1 + \expect[\sum_{j \neq i} \truelabels_j \mid \vect{\truelabels}]}] \tag{masking model}\\
    &= \expect[ \frac{\truelabels_i}{1 + \sum_{j \neq i} \truelabels_j}].
\end{align}

\begin{remark}
Note that this is different from taking the naive formula and replacing
$\truelabels \leftarrow \obslabels/\propensity$, which results in an
expression that is not an upper bound:
\begin{equation}
    \frac{\obslabels_i/\propensity_j}{\sum_{j=1}^{\numlabels}\obslabels_j/\propensity_j} = \frac{\obslabels_i/\propensity_j}{\propensity_i^{-1} + \sum_{j \neq i}\obslabels_j/\propensity_j}.
\end{equation}
\end{remark}

The value of $T^*_i$ can appear in the expression for the loss function either
with positive or negative prefactor. 
Thus, an upper bound for \eqref{eq:ubloss-normalized-nondecomposable} is given by
\begin{equation}
     \tilde{f}(\vect{y}, \hat{\vect{y}}) \coloneqq h(\hat{\vect{y}}) + \sum_{i=1}^{\numlabels} \frac{y_i / \propensity_i}{1 + \sum_{j=1}^{\numlabels} y_j / \propensity_j} g_i(\hat{\vect{y}}).
\end{equation}

Applying to OvA-N and PAL-N, this results in the two expressions
\begin{align}
    \tilde{f}_{\mathrm{OvA-N}}(\vect{y}, \hat{\vect{y}}) &\coloneqq \sum_{i=1}^{\numlabels} \frac{y_i/ \propensity_i}{1 + \sum_{j=1}^{\numlabels} y_j / \propensity_j} (g_{\mathrm{BC}}(1, \hat{y}_i) - g_{\mathrm{BC}}(0, \hat{y}_i)) + g_{\mathrm{BC}}(0, \hat{y}_i) \\
    \tilde{f}_{\mathrm{PAL-N}}(\vect{y}, \hat{\vect{y}}) &\coloneqq \sum_{i=1}^{\numlabels} \frac{y_i / \propensity_i}{1 + \sum_{j=1}^{\numlabels} y_j / \propensity_j} g_{\mathrm{MC}}(i, \hat{\vect{y}}) \label{eq:pal-upper-bound}.
\end{align}
Compared to the unbiased estimates, these expressions have the advantage of
not incurring computational overhead, and are advantageous also from the
perspective of variance.


\FloatBarrier
\section{Evaluation Experiments}
In this section, we validate the result of
\autoref{thm:efficient-multilabel} experimentally using the example of recall.
We first provide a look at synthetic data, where propensities are exactly
known and we can calculate reference ground-truth values. This will show that
even in an ideal setting, if the propensities get too low, the variance of the
estimate increases so much that it can become unusable. Then we apply the
estimator to real predictions, demonstrating that it is applicable for some
datasets, but useless for others.

Consider a setting in which there are 100 different labels, which are
independent and each has a probability of 10\%. We randomly draw \num{10000}
ground-truth label vectors, and generate observed labels by removing according
to a propensity $p$ that is identical for all labels. The predictions are
generated by randomly choosing a label from the ground-truth. We calculate the
average per-example recall using the vanilla estimator, the unbiased
estimator, and the upper bound, and plot the results in \autoref{fig:simdata}.

\begin{figure}
    \centering
    \input{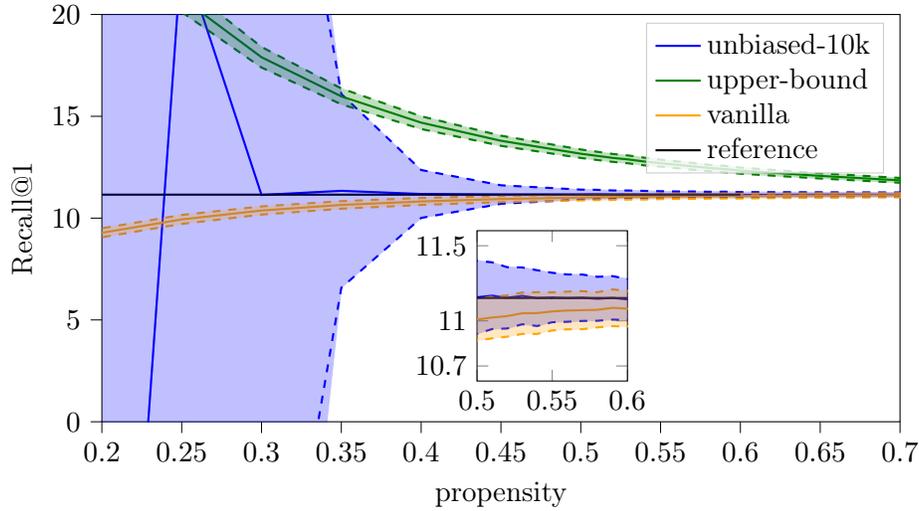}
    \caption{Unbiased estimate of per-example recall with artificial 
    data as described in the main text. The shaded region corresponds to one
    standard deviation, estimated over 100 repetitions. The black line denotes
    the true recall. }
    \label{fig:simdata}
\end{figure}

The figure shows that for propensities lower than \num{0.4}, the unbiased
estimate's variance becomes too large. The upper bound also becomes
unsuitable, because it deviates strongly from the actual value, and the
(biased) vanilla estimate ends up closest to the true value. For medium
propensity, the unbiased estimator still has little variance, but noticeable
better accuracy than the vanilla one. Together with the uniqueness
\autoref{thm:multi-unique}, these results suggest that for datasets with very
low propensity, unbiased estimates are ill-suited to calculate the per-example
recall.

\begin{table*}
    \centering
      \caption{Propensity-scored \eqref{eq:ps-rec} and vanilla
    \eqref{eq:def:recall} recall at $k$ for DiSMEC-style models trained using
    the squared hinge convex surrogate \citet{qaraei_convex_2021}. The Filter
    column indicates the fraction of data points that had to be removed as
    outliers to decrease the variance to reasonable levels. }
    \begin{filecontents*}{plots/psrecall.txt}
{Dataset}           {PsTrainPsMetric1} {PsTrainPsMetric3} {PsTrainPsMetric5} {PsTrainVnMetric1} {PsTrainVnMetric3} {PsTrainVnMetric5} {VnTrainPsMetric1}  {VnTrainPsMetric3}  {VnTrainPsMetric5} {VnTrainVnMetric1} {VnTrainVnMetric3} {VnTrainVnMetric5} {Filter}
{Bibtex}            42.1     65.9     78.3     33.6     53.0     62.0     42.9     67.9     78.4     34.1     54.3     62.5          0\%
{Mediamill}         21.6     49.2     61.9     23.3     52.6     65.6     21.5     49.1     61.5     23.1     52.5     65.3             0\%
{RCV1-2K}               42.4     78.1     82.3     40.4     74.4     80.7     41.5     78.8     84.4     39.8     73.6     79.9       1\%
{AmazonCat-13K}         24.2     57.1     72.3     26.4     59.7     75.0     23.7     55.4     72.2     26.5     59.9     75.2    0.05\%
{AmazonCat-14K}         38.8     68.1     83.0     39.4     69.3     83.9     38.2     67.2     82.7     39.5     69.6     84.1    0.01\%
\end{filecontents*}
\pgfplotstabletypeset[every head row/.style={
            before row={
            \toprule
            & \multicolumn{3}{c|}{PsRec@k}& \multicolumn{3}{c|}{Rec@k} & Filter \\
            },
            after row={\midrule}
        },
        every last row/.style={after row=\bottomrule},
        columns={Dataset,PsTrainPsMetric1,PsTrainPsMetric3,PsTrainPsMetric5,PsTrainVnMetric1,PsTrainVnMetric3,PsTrainVnMetric5,Filter},
        columns/Dataset/.style={string type, column type={l|}, column name={Dataset/k}},
        columns/PsTrainPsMetric1/.style={column name={1}, precision=1, fixed zerofill},
        columns/PsTrainPsMetric3/.style={column name={3}, precision=1, fixed zerofill},
        columns/PsTrainPsMetric5/.style={column name={5}, precision=1, fixed zerofill, column type={c|}},
        columns/PsTrainVnMetric1/.style={precision=1, fixed zerofill, column name={1}},
        columns/PsTrainVnMetric3/.style={precision=1, fixed zerofill, column name={3}},
        columns/PsTrainVnMetric5/.style={precision=1, fixed zerofill, column name={5}, column type={c|}},
        columns/Filter/.style={string type, column name={}},
    ]{plots/psrecall.txt}
    \label{tab:psrec}
\end{table*}
\label{sec:ev-psrec}
Looking at the results presented in \citet{repo}, one notices that recall
metrics are conspicuously absent. This is presumably because they are not
straightforward to compute. In \citet{Jain16} there is a section that argues
how to calculate this in cases where the total number of labels is available,
but for other cases this paper leaves a gap, which is filled by
\autoref{thm:efficient-multilabel}. Unfortunately, as the results with synthetic data
suggest, once the propensities become too low, the method becomes unusable.
This precludes its application to datasets like \texttt{Amazon-670K} or 
\texttt{EURLex-4K}.

To generate the predictions, we trained a DiSMEC \citep{dismec} model using a
convex surrogate unbiased loss function \citep{qaraei_convex_2021} for the
squared-hinge loss. Two difficulties arise when when calculating the unbiased
recall estimate: First, even if the average number of relevant labels is low,
there still can be some samples with a high number of true labels, which
becomes prohibitively expensive in computation. This can be handled by a
sampling approach. The second problem is that the unbiased estimate comes at
the cost of vastly increased variance. In fact, we observed that the mean is
dominated by very few outlier samples, and obtained nonsensical values.
Therefore, we filtered out values in the lowest and highest quantiles. This
results in a biased estimate, but gives much reduced variance. Details can be
found in supplementary \ref{ssec:app:pseval}.

\autoref{tab:psrec} shows that the recall values for vanilla and unbiased
estimation are close except for the Bibtex dataset. A  possible reason could
be that  the propensity model, which was empirically found for very large
datasets, does not fit appropriately for this small dataset. For other
datasets on the repository \citep{repo}, \texttt{Amazon-670K},
\texttt{WikiLSHTC-325K} and \texttt{EURLex-4K}, the variance in the unbiased
estimate is too large to get a meaningful result.

\section{Training Experiments}

Ideally, we would benchmark our loss functions on a task based on real data.
However, for those these we neither know the exact propensities, nor can we
validate that the unbiased estimates and upper bounds produce reasonable
results, since the fully-labeled ground truth is unknown.

\subsection{Experimental Setup}
Instead of using fully artificial data, we have chosen to construct a dataset
based on existing realistic data: We take the \texttt{AmazonCat-13k} data and
consider only the 100 most common labels, which are the ones with the highest
propensity according to the \citet{Jain16} model. We artificially remove
labels according to inverse propensity, which increases linearly based on the
ordering of label frequencies, such that the most common label has an inverse
propensity of two and the 100th most common one has an inverse propensity of
20. This process partially preserves the strong imbalances that are typical of
extreme classification datasets.

On this data, we perform the following experiment: We train a linear
classifier with $L_2$-regularization using different basis loss functions with
\textbf{a)} the vanilla version of the loss on clean training data and
\textbf{b)} noisy training data, as well as \textbf{c)} the unbiased version
and \textbf{d)} the upper-bound version on noisy data. For each training run,
we evaluate training loss, as well as the task losses precision and recall at
k, on noisy and clean training and test data. For the evaluation on noisy
data, the corresponding unbiased estimators are used.

By training with the vanilla loss, we get an upper bound how well the given
network could do without noise (config \textbf{a}) and how poorly it would
perform without any mitigations (config \textbf{b}). In the first setting, we
expect the calculated loss on clean and noisy data to match, since the model
cannot overfit to the specific noise pattern in the training data, whereas in
configurations \textbf{b)} to \textbf{d)} we expect to see a difference. As we
remove labels independently of features, the unbiased estimate on noisy test
data should be equal to the actual value on clean test data, if the test set
contains sufficient samples to result in accurate estimates.

We took 30\% of the original training data and used them as validation data to
determine the optimal value for the strength of $L_2$-regularization. As the
choice of loss function influences the bias-variance trade-off, this value
needs to be determined for each configuration. Note that when we train with
noisy data, we also use noisy data for validation, \ie we assume a setting in
which no clean data is available at all.

The network is optimized using Adam \citep{kingma2017adam} with an initial
learning rate of $10^{-4}$ for the first $15$ epochs and $10^{-5}$ for the
remaining five epochs, with a mini-batch size of 512.

\subsection{Overfitting to Sample and Noise Pattern}
\begin{figure}[tb]
    \centering
    \input{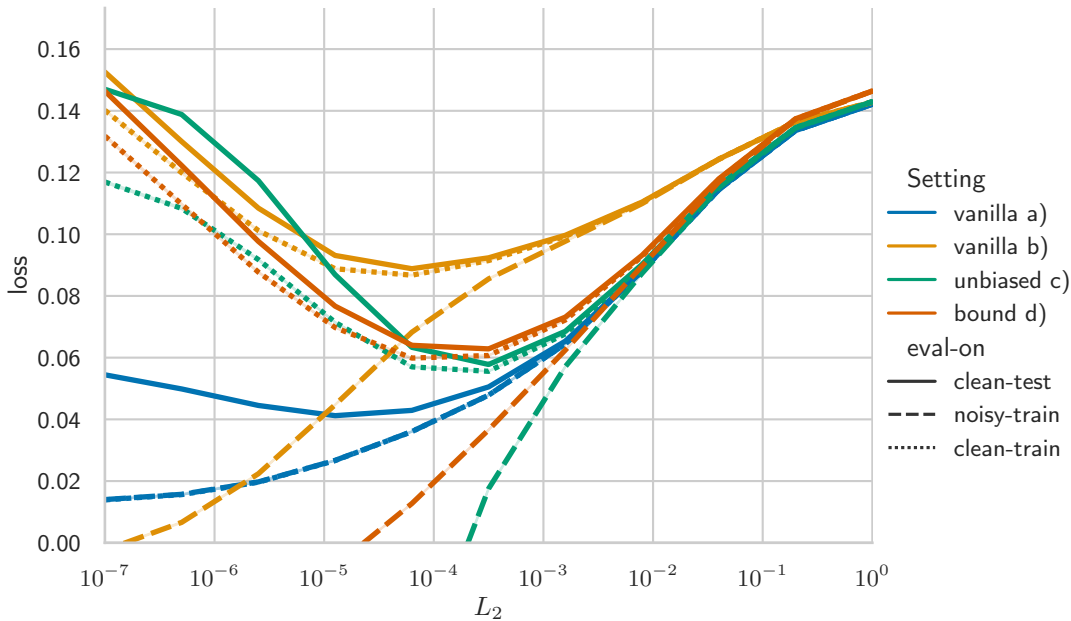}
    \caption{Binary cross-entropy for different regularization
    strengths, evaluated on noisy training data, clean training data, and
    clean test data. The gap between the dashed and the dotted lines
    corresponds to the overfitting to the noise pattern, the much smaller gap
    between dotted and solid lines shows the additional generalization gap due
    to the finite training sample.}
    \label{fig:bce-overfitting}
\end{figure}

In addition to the overfitting that is due to the finiteness of the training
data, the missing-labels setting causes additional overfitting because only a
single realization of the noise pattern is observed.

Let $h$ be a classifier that depends on the observed data $\obslabels$, and
$f^*$ a loss function with estimator $f = \pswt(f^*)$. Then the generalization
decomposes into
\begin{equation}
    \trisk_{f*}[h] - \oerisk_{f}[h] = \underbrace{\trisk_{f*}[h] - \terisk_{f*}[h]}_{\text{finite sample}} + \underbrace{\terisk_{f*}[h] - \oerisk_{f}[h]}_{\text{noise pattern}},
\end{equation}
the difference between the true risk $\trisk_{f*}[h]$  and the empirical risk
on clean training data $\terisk_{f*}[h]$, and the difference between that and
the estimated empirical risk on observed data $\oerisk_{f}[h]$. Because the
classifier $h$ depends (through $\obslabels = \vect{M} \odot \truelabels$) on
the mask variables, $f$ does not give an unbiased estimate (on training data)
and thus the second term is non-zero even in expectation.

In fact, in the linear classifier experiment described here, the noise-pattern
overfitting is much stronger than the overfitting due to finite sampling. 
\autoref{fig:bce-overfitting} shows this for the case of the BCE loss
in OvA-reduction, though the same effect can be seen also for the other loss
functions, see supplementary. For the classifier trained on clean data (blue),
the weights are independent of the noise pattern and thus the dashed and
dotted lines coincide in expectation.

We can see that the label noise acts as an implicit regularizer, in the sense
that for low regularization, the gap between clean test and clean training
data is much lower when the classifier is trained on noisy data as compared to
clean training data.

\subsection{Results for Loss Minimization}
\begin{figure}[p]
\centering
\input{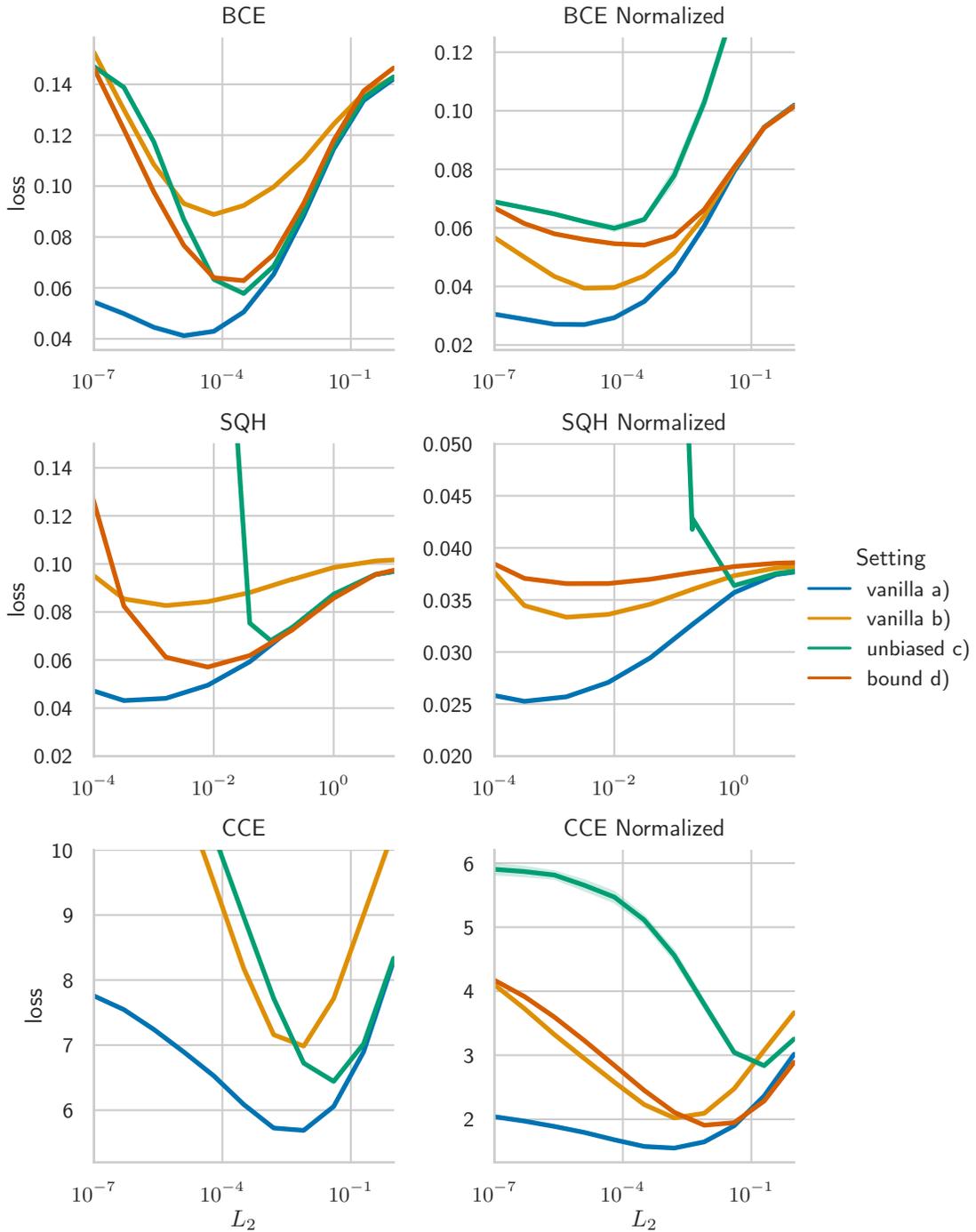}
\caption{Comparison of different loss functions, evaluated on clean test data,
in combination with different schemes for addressing missing labels. Details
are in the main text.}
\label{fig:test-loss-over-reg}
\end{figure}

As \autoref{fig:bce-overfitting} shows, for the case of OvA reduction using
the BCE loss, the training loss gets reduced much further using the unbiased
loss function or the upper-bound loss function than using the vanilla loss.
This decrease more than compensates the increase in generalization gap, and as
such the minimal loss, \ie the loss at optimal regularization, is better with
these two variants of the loss function.

Does this improvement also occur for other reductions and loss functions? This
question is addressed by \autoref{fig:test-loss-over-reg}, which plots the
loss on clean test data over varying $L_2$ regularization strength for
different base losses.

The first row of graphs shows BCE-based losses, with the left graph depicting
the same data as \autoref{fig:bce-overfitting}. The right graph is for the
normalized reduction, where the increased overfitting exceeds the improvement
in training loss, and overall the unbiased training produces worse results
than just training with the vanilla loss. We have also included an evaluation
using the upper-bound formula \eqref{eq:upper-bound-for-T}. Note that for the
normalized BCE loss, that formula does not result in an upper bound, but the
graph shows that it empirically works better than the unbiased loss, yet still
worse than vanilla.

A surprising feature of the normalized BCE result is that even for very strong
regularization, the unbiased estimate underperforms vanilla loss. We attribute
this to the non-convexity of \eqref{eq:ova-n-unbiased} because the increase in
loss is due to large training loss as opposed to the generalization gap. By
choosing different initial weights (see supplementary
\ref{seq:supp-norm-bce}), we can reduce the training loss for unbiased
training to that of vanilla training, supporting the hypothesis that
sub-optimal minima are the problem. Compared to the decomposable case (which
is also non-convex), the normalized case requires scaling by products of
several inverse propensities, and thus induces much larger prefactors in front
of the terms that cause the non-convexity, which might explain why the same
phenomenon does not occur in the decomposable reduction.

The second row shows squared-hinge based losses, which are also subject to the
One-vs-All reduction. The main difference between this loss and BCE is that
squared-hinge is even more susceptible to overfitting in the unbiased case.
Whereas for BCE the unbiased and upper-bound variations give roughly
comparable results, in this case the upper-bound loss proved to be much more
stable. In the normalized setting, the unbiased loss becomes very large for
even larger regularization strength ($\approx 10^{-1}$) compared to the
decomposable variant ($\approx 10^{-2}$). As in the BCE-normalized setting,
the expression in \eqref{eq:upper-bound-for-T} is not actually an upper bound,
and we can see empirically that it does not work well here. We again can
observe that the vanilla loss performs better than the unbiased loss.

The third row shows categorical cross-entropy, which differs from the other two
rows in that it results in a Pick-all-Labels reduction. In the non-normalized
case, the upper bound and unbiased loss are identical and outperform the
vanilla loss. The normalized case shows the same problems as above for the
unbiased loss, but the upper bound \eqref{eq:pal-upper-bound} successfully
improves on the vanilla loss.

In terms of the bias-variance trade-off, the graphs show a clear trend: The
optimal regularization for training on noisy data is larger than on clean
data. It is also larger when using the unbiased or upper-bound loss as
compared to vanilla loss.

\subsection{Task Losses}
\begin{figure}[tb]
    \centering
    \input{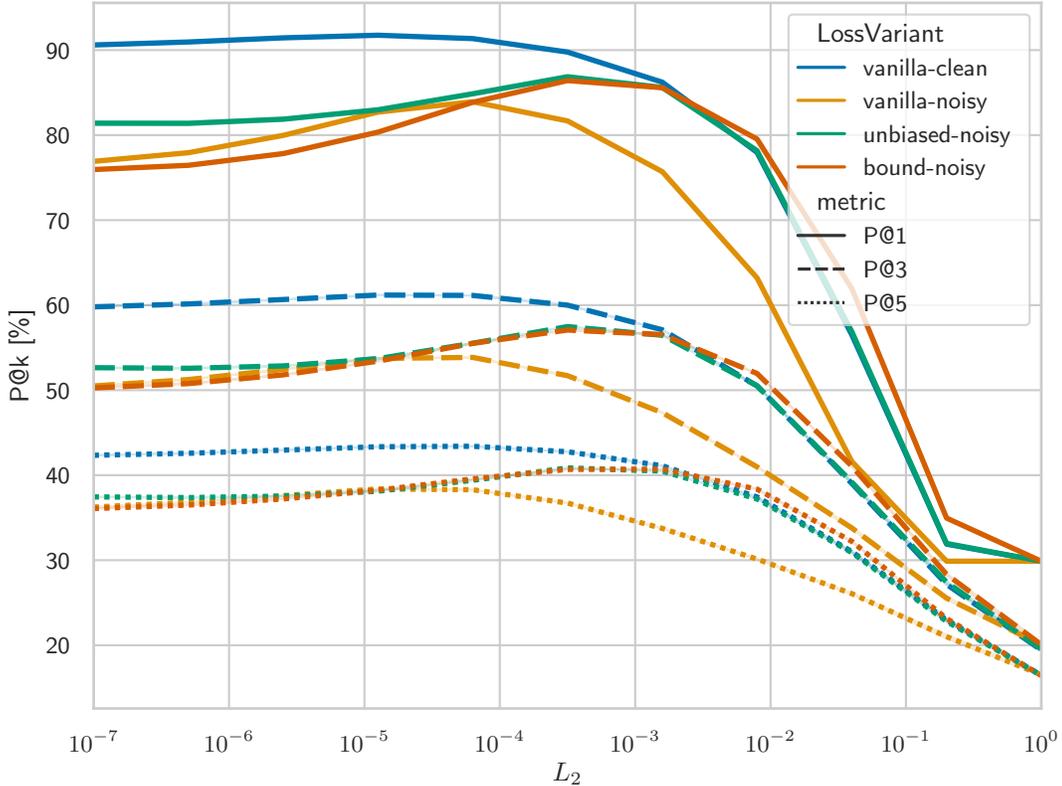}
    \caption{Precision at $k$ for training with binary cross-entropy
    (normalized reduction) for the different loss variations.}
    \label{fig:bce-Patk}
\end{figure}

Even though the training process on clean data is based on the optimization of
a decomposition into a differentiable loss, this is typically not the
quantity that is ultimately of interest. Instead, what we really want is a
maximization of precision or recall at the top. 

The behaviour of precision at k in dependence of the regularization is
depicted in \autoref{fig:bce-Patk} for the BCE loss. There are some notable
differences in the behaviour of the P@k metrics compared to the loss function:
Whereas the increased overfitting for low regularization results in a strong
deterioration in the loss function, the corresponding decrease in precision
metrics is relatively mild. For the loss function, the unbiased estimate works
better at higher regularization and worse at lower regularization than the
upper bound, but here the situation is reversed.

\begin{figure}[tb]
    \centering
    \input{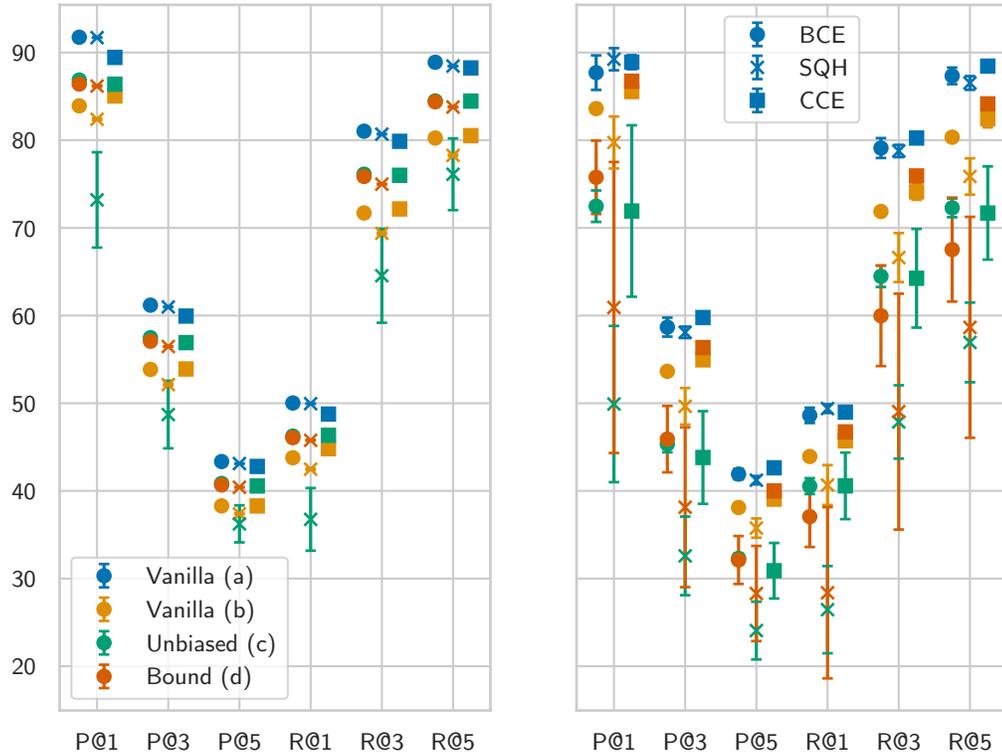}
    \caption{Precision and recall at the optimal (for loss minimization) regularization 
    strength for decomposable (left) and normalized (right) training losses.
    These values have been calculated on clean test data. The error intervals
    denote standard deviation and have been determined using 5 runs.}
    \label{fig:PRatK}
\end{figure}

The full results for precision and recall at the optimal regularization
parameter are presented in \autoref{fig:PRatK}. In this context, optimal is to
be understood as the value for which the unbiased estimate of the loss
function on noisy validation data is minimized. These values might be slightly
different than the optimal values for a specific P@k or R@k metric. The same
data, along with the corresponding value of the loss function and the optimal
regularization parameter, has been summarized in
\autoref{tab:amazoncat-results-recall}.

Comparing the decomposable and the normalized settings, we can see that the
normalized losses induce larger fluctuations (measured in standard deviation) for
the OvA settings and the unbiased PAL setting, with the exception of vanilla
training on noisy data (b) for BCE loss. This holds across all six task
losses. Note that this cannot be explained by the variance in the unbiased
estimators, as the values presented here have been calculated directly on the
clean ground-truth test data.

The graph also shows that the ordering of settings a)-d) is mostly stable
(for fixed loss function) over the six different metrics, but differs when
switching loss function. For OvA-BCE, the unbiased and upper-bound losses 
c) and d) are almost equal, but for OvA-SQH the unbiased performs far worse,
and with much larger fluctuations, which is consistent with the observations
in \autoref{fig:test-loss-over-reg}. For PAL-CCE, the two settings are the
same.

In the normalized cases, we can see that the OvA-N reduction results in
mostly the same ordering for both BCE and SQH: The vanilla loss b) is better
than applying \eqref{eq:pal-upper-bound} d), which in turn is better than the
unbiased loss c). The exception are the recall metrics for BCE, where the
unbiased loss appears to give better results.

For the PAL-N reduction we can see that the upper bound d) does result in
better task loss than using vanilla loss b), but as with the other normalized
setups the unbiased loss c) performs worse.

\begin{table}[tb]
\caption{
Training results on modified AmazonCat-13K data for using different loss
functions in their \textbf{V}anilla, \textbf{U}nbiased or
Upper-\textbf{B}ounded variants. BCE denotes the binary cross-entropy and SQH
the squared-hinge loss corresponding to a OvA decomposition, and CCE denotes
(softmax) categorical cross-entropy corresponding to a PaL decomposition. The
settings marked with a* denote reference runs on clean data.}
\label{tab:amazoncat-results-recall}
\centering

\pgfplotstabletypeset[
    col sep=comma,
    every head row/.style={
        before row={
        \toprule
        Setting & \multicolumn{3}{c|}{Precision}& \multicolumn{3}{c|}{Recall} & Loss & Reg. \\},
        after row={\midrule}
    },
    every last row/.style={after row=\bottomrule},
    columns={Setting,
        clean-test/P@1/mean,clean-test/P@3/mean,clean-test/P@5/mean,
        clean-test/R@1/mean,clean-test/R@3/mean,clean-test/R@5/mean,
        clean-test/loss/mean,config/l2_reg/mean},
    columns/Setting/.style={string type, column type={l|}, column name={}},
    columns/clean-test/P@1/mean/.style={column name={P@1}, precision=1,fixed,fixed zerofill},
    columns/clean-test/P@3/mean/.style={column name={P@3}, precision=1,fixed,fixed zerofill},
    columns/clean-test/P@5/mean/.style={column name={P@5}, precision=1,fixed,fixed zerofill, column type={c|}},
    columns/clean-test/R@1/mean/.style={column name={R@1},  precision=1,fixed,fixed zerofill},
    columns/clean-test/R@3/mean/.style={column name={R@3},  precision=1,fixed,fixed zerofill},
    columns/clean-test/R@5/mean/.style={column name={R@5},  precision=1,fixed,fixed zerofill, column type={c|}},
    columns/clean-test/loss/mean/.style={column name={}, column type={c|}, precision=1},
    columns/config/l2_reg/mean/.style={column name={}},
    every row no 8/.style={before row=\midrule},
    every row no 16/.style={before row=\midrule},
    every row no 1/.style={after row={[.6ex]}},
    every row no 3/.style={after row={[.6ex]}},
    every row no 5/.style={after row={[.6ex]}},
    every row no 9/.style={after row={[.6ex]}},
    every row no 11/.style={after row={[.6ex]}},
    every row no 13/.style={after row={[.6ex]}},
    every row no 17/.style={after row={[.6ex]}},
    every row no 19/.style={after row={[.6ex]}},
    every row no 21/.style={after row={[.6ex]}},
]{plots/decompositions/data.csv}
\end{table}

Finally, we want to know in which circumstances the normalized variations
perform better than the decomposable ones. The results in
\citet{menon_multilabel_2019} prove that in the asymptotic case the
normalized reductions are consistent for recall whereas the others are
consistent for precision. However, as seen above, despite the large number of
training instances in the dataset, overfitting is still a major problem, and
thus it may well be that a method that is consistent for recall also gets
better results in precision if its generalization gap is significantly
smaller.

From \autoref{tab:amazoncat-results-recall} we can see that for the BCE loss
the normalized reduction generally results in worse performance for both
precision and recall, with the notable exception of using vanilla loss on
noisy data, where there is a slight increase in recall. Nonetheless, using
unbiased or upper-bound losses for the decomposable reduction results in
overall better performance. For the squared hinge (SQH) loss, the normalized
reduction performs worse for all variations and all metrics.

For the decomposable CCE reduction, the upper bound is equal to the unbiased
loss. Compared with the bound for the normalized reduction, the latter gives
better values in P@1 and R@1, but worse for the other metrics. The unbiased
normalized reduction performs much worse across all metrics. For the vanilla
loss , the results are as expected according to \citet{menon_multilabel_2019}
with the normalized version yielding better recall. Interestingly, when used
on noisy data, it also gives better precision.

\FloatBarrier

\section{Related Work}
\label{sec:related}
\paragraph{Unbiased Estimates for Noisy Labels}
Learning with missing labels is a specific instance of learning with
class-conditional noise. For the case of binary labels, unbiased estimates of
the loss function can be found in \citet{natarajan_cost-sensitive_2017}. Note
that the generalization bound given therein (Lemma 8) is missing a factor of
max norm of the adapted loss function $\|\tilde{\lossfn}\|_{\infty}$, \cf
\autoref{thm:genbound}.

An even more general approach is given in \citet{van_rooyen_theory_2017}. In
their notation, $f$ is a function and $\mathds{P}$ the probability
distribution over clean data, that is transformed by the invertible operator
$\mathsf{T}$ into a \emph{corrupted} probability distribution. Let
$\mathsf{R}$ be the inverse of $\mathsf{T}$, and $\mathsf{R}^*$ its adjoint,
then it holds
\begin{equation}
    \langle \mathds{P}, f \rangle = \langle \mathsf{R} \circ \mathsf{T}(\mathds{P}), f \rangle = \langle \mathsf{T}(\mathds{P}), \mathsf{R}^*(f) \rangle.
\end{equation}
This equation forms the basis for their \enquote{Theorem 5 (Corruption
Corrected Loss)}, which states that a \emph{corruption corrected} function
$l_{\mathrm{R}}$ is given by
\begin{equation}
    l_{\mathrm{R}}(\cdot, a) = \mathsf{R}^*(l(\cdot, a)) \; \forall a \in \mathcal{A}, \label{eq:van-rooyen-unbiased}
\end{equation}
where $\mathcal{A}$ denotes the set of possible actions that will be evaluated
by the loss functions. For a finite label space with $n$ possible, the
operator $\mathsf{R}^*$ can be represented with an $n \times n$ matrix.

In the specific case of multilabel classification, our results
\autoref{thm:efficient-multilabel} show that out of the $2^l$ possible
values of the label vector that are necessary to evaluate
\eqref{eq:van-rooyen-unbiased} in the general case, in fact only these that
are a subset of the observed labels need to be taken into account, thus
requiring only $2^{\|y\|_0}$ evaluations.

\paragraph{Robust Loss-Functions}
An even more robust approach than choosing a loss function which compensates
for noisy labels is to use a learning algorithm that is inherently noise
tolerant. This has the advantage that one does not need to estimate the noise
rate, and cannot introduce additional error by misspecification. For symmetric
label noise with rates less than 50\%, \citet{ghosh_making_2015} proved that
methods minimizing losses which fulfill a symmetry condition $f(0, \cdot) + f(1,
\cdot) = c$ for some constant $c$ are noise tolerant. Certain performance
objectives such as the balanced error or the AUC are noise robust even under
the more general setting of mutually contaminated distributions as shown in
\citet{menon_learning_2015}.

\paragraph{Data Re-Calibration}
A data re-calibration approach tries to identify from the training data which
samples are corrupted
\citep{NEURIPS2018_a19744e2,zheng2020error,pmlr-v80-jiang18c}. The co-teaching
approach \citep{NEURIPS2018_a19744e2} maintains two interacting networks: For
each minibatch, each network selects a subset of examples with low loss value,
which is assumed to indicate that these are clean instances. The weights of
the other network are then updated by training only on these selected
examples. Given that deep neural networks have been observed to initially fit
clean data and start overfitting on noise as the training process progresses
\citep{arpit_closer_2017}, they decrease the selected fraction of the
minibatch over time.

Some theoretical justification for these approaches is provided by
\citet{zheng2020error} in cases where the probability mass for instances very
close to the decision boundary is low. For noisy labels with transition
probabilities that are independent of the features, a sufficiently accurate
model for predicting the true class-conditional probability $\eta(\datapvar) =
\prob[\obslabels=1\mid\datapvar]$ can identify mislabeled samples. Based on
that, they developed a likelihood-ratio test to decide whether a label in the
training data should be flipped. Despite the theoretical foundation of their
approach, they still need some empirical adjustments to make the method work
in practice, \eg they introduced an additional \emph{retroactive loss} term in
order to stabilize training.

\paragraph{Post-Processing}
It is also possible to first train a scorer on the noisy data naively, from
which a classifier adapted to a given rate of missing labels can be
constructed by choosing an appropriate threshold. For a naive scorer that
predicts the class probabilities for each data point, the corrected threshold
is given in \citet{menon_learning_2015}. Similarly, the inference procedure of
probabilistic label trees can be adapted to take into account a propensity
model \citep{propensity-plt}.

Other approaches are to try to design new losses specifically tailored to
deal with missing labels, such as a group-lasso based formulation
\citep{bucak2011multi}.

\paragraph{Positive-Unlabeled Learning}
Learning with missing labels is highly related to the problem of learning from
positive and unlabeled (PU) data. This can be interpreted in two ways, the
\emph{censoring} setting which is identical to learning with missing labels,
and the \emph{case-control} setting in which the positive labels are drawn
independently from the unlabeled data \citep{elkan_learning_2008}. In
the latter case the marginal of the true labels of the training and test data
might be different. In that setting, instead of a noise rate, the class prior
$\pi$ needs to be known (or estimated), then a corrected loss function can be
determined as in \citet{du_plessis_convex_2015}. The appearing difficulties,
that non-negativity and convexity need not be preserved in the new loss, are
the same as in our setting
\citep{kiryo_positive-unlabeled_2017,chou_unbiased_2020}.

\paragraph{Semi-Supervised Learning}
A slightly different setting with missing labels is given by semi-supervised
learning. Here, for each example the values of only a (known) subset of the
labels are available, that is label can be one of three values 1 (positive),
-1 (negative), and 0 (unknown). If the loss function decomposes over labels,
then one strategy for coping with this situation, taken in
\citet{yu2014large}, is to only sum up the contributions where the label is
known, \ie the unknown labels are masked out.

\section{Summary and Discussion}

We have shown that the modelling of missing label learning problems using a
mask variable provides an easy way of deriving unbiased estimators for both
the binary and the multilabel setting, if labels go missing independently.
These unbiased estimators are unique, and may show undesired properties:
Even if the original loss function was convex and lower-bounded, the estimate 
can be non-convex and unbounded. Even in a pure evaluation setting, where these
properties are not required, increasing variance as the propensity decreases
poses a significant problem and may preclude the use of unbiased estimates.

As a mitigation, we propose to use convex upper-bounds. For the binary case we
can write down a general solution given in \eqref{eq:binary-upper-bound}. In
the multilabel setting, we have considered four important cases that arise out
of the multilabel reductions given in \citet{menon_multilabel_2019}.
Particularly favourable among them is the Pick-all-Labels reduction, as it
directly leads to a convex function. For its corresponding normalized form, we
still can construct a convex upper-bound. In the One-vs-All case without
normalization, the binary convex upper-bound can be applied, but finding a
bound with normalization is still an open problem. An overview of the reductions
is given in \autoref{tab:reduction-overview}.

\begin{table}
\caption{Overview of multilabel loss reductions. }
\label{tab:reduction-overview}
\centering
\begin{tabular}{l|ccc}
    \toprule
    Reduction       & Base          & Consistency   & Convexity \\ \midrule
    One-vs-All      & Binary        & Precision     & Upper-Bound \\
    OvA-Norm        & Binary        & Recall        & \textbf{?} \\
    Pick-all-Labels & Multiclass    & Precision     & Yes \\
    PaL-Norm        & Multiclass    & Recall        & Upper-Bound \\ \bottomrule
\end{tabular}
\end{table} 

These results suggest that asymptotically, PaL reductions are preferable over
OvA reductions. In practice, however, the situation is less clear. In our
experiments we observed that unnormalized OvA produced the best results in
terms of recall, even though this loss is in fact not consistent for recall.
The most clear recommendation that can be drawn from our results is that if
you want to use a normalized reduction, PaL-Norm is to be preferred over
OvA-Norm because we currently lack a convex surrogate for the latter.

As suggested by theoretical results (\autoref{thm:genbound}) and corroborated
by the experiments, missing labels lead to a shift in the bias-variance
trade-off. If the data has missing labels, more regularization is required,
irrespective of whether training uses vanilla-, unbiased-, or convex
upper-bound losses. Looking more closely at the overfitting phenomenon, we
found that the generalization error can be split into two parts: the
difference between the empirical errors on the noisy and the clean (finite)
data, and the difference between the clean empirical error and the true risk.
We found that the overfitting to the specific noise pattern substantially
exceeded the overfitting to the finite sample. Our findings agree with the
observation of \citet{arpit_closer_2017} which found that typical regularizers
prevent a deep network from memorizing noisy examples, while not hindering 
the learning of patterns from clean instances. 

Due to the uniqueness results, the problems mentioned above are unavoidable
when using unbiased estimates. This suggests that further research should look
into loss functions that allow tuning the trade-off between bias and variance.
Having a slight bias in the loss function would usually not be problematic,
especially considering that the propensity values which we have assumed to be
given in this paper will in practice actually be only estimates, so that the
unbiased estimates derived here will also have a slight bias due to
misspecification.

\section*{Acknowledgements}
We would like to thank Krzysztof Dembczynski, Marek Wydmuch, Mohammadreza
Qaraei, and Thomas Staudt for discussions and feedback on earlier drafts of
this paper.

\bibliographystyle{icml2021}
\bibliography{lit}

\begin{thebibliography}{36}
\providecommand{\natexlab}[1]{#1}
\providecommand{\url}[1]{\texttt{#1}}
\expandafter\ifx\csname urlstyle\endcsname\relax
  \providecommand{\doi}[1]{doi: #1}\else
  \providecommand{\doi}{doi: \begingroup \urlstyle{rm}\Url}\fi

\bibitem[Agrawal et~al.(2013)Agrawal, Gupta, Prabhu, and Varma]{Agrawal13}
Rahul Agrawal, Archit Gupta, Yashoteja Prabhu, and Manik Varma.
\newblock Multi-label learning with millions of labels: Recommending advertiser
  bid phrases for web pages.
\newblock In \emph{Proceedings of the 22nd International Conference on World
  Wide Web}, WWW '13, page 13–24, New York, NY, USA, 2013. Association for
  Computing Machinery.
\newblock ISBN 9781450320351.
\newblock \doi{10.1145/2488388.2488391}.
\newblock URL \url{https://doi.org/10.1145/2488388.2488391}.

\bibitem[Arpit et~al.(2017)Arpit, Jastrz{\k{e}}bski, Ballas, Krueger, Bengio,
  Kanwal, Maharaj, Fischer, Courville, Bengio, and
  Lacoste-Julien]{arpit_closer_2017}
Devansh Arpit, Stanis{\l}aw Jastrz{\k{e}}bski, Nicolas Ballas, David Krueger,
  Emmanuel Bengio, Maxinder~S. Kanwal, Tegan Maharaj, Asja Fischer, Aaron
  Courville, Yoshua Bengio, and Simon Lacoste-Julien.
\newblock A {Closer} {Look} at {Memorization} in {Deep} {Networks}.
\newblock In Doina Precup and Yee~Whye Teh, editors, \emph{Proceedings of the
  34th {International} {Conference} on {Machine} {Learning}}, volume~70 of
  \emph{Proceedings of Machine Learning Research}, pages 233--242. PMLR, 06--11
  Aug 2017.

\bibitem[Babbar and Sch\"{o}lkopf(2017)]{dismec}
Rohit Babbar and Bernhard Sch\"{o}lkopf.
\newblock Dismec: Distributed sparse machines for extreme multi-label
  classification.
\newblock In \emph{WSDM}, pages 721--729, 2017.

\bibitem[Babbar and Schölkopf(2019)]{babbar_data_2019}
Rohit Babbar and Bernhard Schölkopf.
\newblock Data scarcity, robustness and extreme multi-label classification.
\newblock \emph{Machine Learning}, 108\penalty0 (8):\penalty0 1329--1351,
  September 2019.
\newblock ISSN 1573-0565.
\newblock \doi{10.1007/s10994-019-05791-5}.
\newblock URL \url{https://doi.org/10.1007/s10994-019-05791-5}.

\bibitem[Bhatia et~al.(2016)Bhatia, Dahiya, Jain, Prabhu, and Varma]{repo}
Kush Bhatia, Kunal Dahiya, Himanshu Jain, Yashoteja Prabhu, and Manik Varma.
\newblock The extreme classification repository: Multi-label datasets and code.
\newblock \url{http://manikvarma.org/downloads/XC/XMLRepository.html}, 2016.

\bibitem[Bucak et~al.(2009)Bucak, Mallapragada, Jin, and
  Jain]{bucak2009efficient}
Serhat~S Bucak, Pavan~Kumar Mallapragada, Rong Jin, and Anil~K Jain.
\newblock Efficient multi-label ranking for multi-class learning: application
  to object recognition.
\newblock In \emph{2009 IEEE 12th International Conference on Computer Vision},
  pages 2098--2105. IEEE, 2009.

\bibitem[Bucak et~al.(2011)Bucak, Jin, and Jain]{bucak2011multi}
Serhat~Selcuk Bucak, Rong Jin, and Anil~K Jain.
\newblock Multi-label learning with incomplete class assignments.
\newblock In \emph{CVPR 2011}, pages 2801--2808. IEEE, 2011.

\bibitem[Chou et~al.(2020)Chou, Niu, Lin, and Sugiyama]{chou_unbiased_2020}
Yu-Ting Chou, Gang Niu, Hsuan-Tien Lin, and Masashi Sugiyama.
\newblock Unbiased {Risk} {Estimators} {Can} {Mislead}: {A} {Case} {Study} of
  {Learning} with {Complementary} {Labels}.
\newblock In \emph{International {Conference} on {Machine} {Learning}}, pages
  1929--1938. PMLR, November 2020.
\newblock URL \url{https://proceedings.mlr.press/v119/chou20a.html}.
\newblock ISSN: 2640-3498.

\bibitem[Dahiya et~al.(2021)Dahiya, Agarwal, Saini, Gururaj, Jiao, Singh,
  Agarwal, Kar, and Varma]{Dahiya21b}
K.~Dahiya, A.~Agarwal, D.~Saini, K.~Gururaj, J.~Jiao, A.~Singh, S.~Agarwal,
  P.~Kar, and M~Varma.
\newblock Siamesexml: Siamese networks meet extreme classifiers with 100m
  labels.
\newblock In \emph{Proceedings of the International Conference on Machine
  Learning}, July 2021.

\bibitem[Dekel and Shamir(2010)]{dekel2010multiclass}
Ofer Dekel and Ohad Shamir.
\newblock Multiclass-multilabel classification with more classes than examples.
\newblock In \emph{Proceedings of the Thirteenth International Conference on
  Artificial Intelligence and Statistics}, pages 137--144, 2010.

\bibitem[Deng et~al.(2010)Deng, Berg, Li, and Fei-Fei]{deng2010does}
Jia Deng, Alexander~C Berg, Kai Li, and Li~Fei-Fei.
\newblock What does classifying more than 10,000 image categories tell us?
\newblock In \emph{ECCV}, 2010.

\bibitem[Du~Plessis et~al.(2015)Du~Plessis, Niu, and
  Sugiyama]{du_plessis_convex_2015}
Marthinus Du~Plessis, Gang Niu, and Masashi Sugiyama.
\newblock Convex formulation for learning from positive and unlabeled data.
\newblock In \emph{International conference on machine learning}, pages
  1386--1394, 2015.

\bibitem[Elkan and Noto(2008)]{elkan_learning_2008}
Charles Elkan and Keith Noto.
\newblock Learning classifiers from only positive and unlabeled data.
\newblock In \emph{Proceedings of the 14th {ACM} {SIGKDD} international
  conference on {Knowledge} discovery and data mining}, {KDD} '08, pages
  213--220, New York, NY, USA, August 2008. Association for Computing
  Machinery.
\newblock ISBN 978-1-60558-193-4.
\newblock \doi{10.1145/1401890.1401920}.
\newblock URL \url{https://doi.org/10.1145/1401890.1401920}.

\bibitem[Ghosh et~al.(2015)Ghosh, Manwani, and Sastry]{ghosh_making_2015}
Aritra Ghosh, Naresh Manwani, and P.~S. Sastry.
\newblock Making risk minimization tolerant to label noise.
\newblock \emph{Neurocomputing}, 160:\penalty0 93--107, 2015.
\newblock Publisher: Elsevier.

\bibitem[Guo et~al.(2019)Guo, Mousavi, Wu, Holtmann-Rice, Kale, Reddi, and
  Kumar]{guo_breaking_2019}
Chuan Guo, Ali Mousavi, Xiang Wu, Daniel~N. Holtmann-Rice, Satyen Kale, Sashank
  Reddi, and Sanjiv Kumar.
\newblock Breaking the {Glass} {Ceiling} for {Embedding}-{Based} {Classifiers}
  for {Large} {Output} {Spaces}.
\newblock \emph{Advances in Neural Information Processing Systems}, 32, 2019.
\newblock URL
  \url{https://proceedings.neurips.cc/paper/2019/hash/78f7d96ea21ccae89a7b581295f34135-Abstract.html}.

\bibitem[Han et~al.(2018)Han, Yao, Yu, Niu, Xu, Hu, Tsang, and
  Sugiyama]{NEURIPS2018_a19744e2}
Bo~Han, Quanming Yao, Xingrui Yu, Gang Niu, Miao Xu, Weihua Hu, Ivor Tsang, and
  Masashi Sugiyama.
\newblock Co-teaching: Robust training of deep neural networks with extremely
  noisy labels.
\newblock In S.~Bengio, H.~Wallach, H.~Larochelle, K.~Grauman, N.~Cesa-Bianchi,
  and R.~Garnett, editors, \emph{Advances in Neural Information Processing
  Systems}, volume~31. Curran Associates, Inc., 2018.
\newblock URL
  \url{https://proceedings.neurips.cc/paper/2018/file/a19744e268754fb0148b017647355b7b-Paper.pdf}.

\bibitem[Jain et~al.(2016)Jain, Prabhu, and Varma]{Jain16}
Himanshu Jain, Yashoteja Prabhu, and Manik Varma.
\newblock Extreme multi-label loss functions for recommendation, tagging,
  ranking and other missing label applications.
\newblock In \emph{KDD}, August 2016.

\bibitem[Jain et~al.(2019)Jain, Balasubramanian, Chunduri, and
  Varma]{jain2019slice}
Himanshu Jain, Venkatesh Balasubramanian, Bhanu Chunduri, and Manik Varma.
\newblock Slice: Scalable linear extreme classifiers trained on 100 million
  labels for related searches.
\newblock In \emph{WSDM}, pages 528--536, 2019.

\bibitem[Jiang et~al.(2018)Jiang, Zhou, Leung, Li, and
  Fei-Fei]{pmlr-v80-jiang18c}
Lu~Jiang, Zhengyuan Zhou, Thomas Leung, Li-Jia Li, and Li~Fei-Fei.
\newblock {M}entor{N}et: Learning data-driven curriculum for very deep neural
  networks on corrupted labels.
\newblock In Jennifer Dy and Andreas Krause, editors, \emph{Proceedings of the
  35th International Conference on Machine Learning}, volume~80 of
  \emph{Proceedings of Machine Learning Research}, pages 2304--2313. PMLR,
  10--15 Jul 2018.
\newblock URL \url{https://proceedings.mlr.press/v80/jiang18c.html}.

\bibitem[Kingma and Ba(2017)]{kingma2017adam}
Diederik~P. Kingma and Jimmy Ba.
\newblock Adam: A method for stochastic optimization, 2017.

\bibitem[Kiryo et~al.(2017)Kiryo, Niu, du~Plessis, and
  Sugiyama]{kiryo_positive-unlabeled_2017}
Ryuichi Kiryo, Gang Niu, Marthinus~C du~Plessis, and Masashi Sugiyama.
\newblock Positive-{Unlabeled} {Learning} with {Non}-{Negative} {Risk}
  {Estimator}.
\newblock In I.~Guyon, U.~V. Luxburg, S.~Bengio, H.~Wallach, R.~Fergus,
  S.~Vishwanathan, and R.~Garnett, editors, \emph{Advances in {Neural}
  {Information} {Processing} {Systems} 30}, pages 1675--1685. Curran
  Associates, Inc., 2017.
\newblock URL
  \url{http://papers.nips.cc/paper/6765-positive-unlabeled-learning-with-non-negative-risk-estimator.pdf}.

\bibitem[{Lapin} et~al.(2018){Lapin}, {Hein}, and {Schiele}]{8036272}
M.~{Lapin}, M.~{Hein}, and B.~{Schiele}.
\newblock Analysis and optimization of loss functions for multiclass, top-k,
  and multilabel classification.
\newblock \emph{IEEE Transactions on Pattern Analysis and Machine
  Intelligence}, 40\penalty0 (7):\penalty0 1533--1554, 2018.
\newblock \doi{10.1109/TPAMI.2017.2751607}.

\bibitem[Menon et~al.(2015)Menon, Van~Rooyen, Ong, and
  Williamson]{menon_learning_2015}
Aditya Menon, Brendan Van~Rooyen, Cheng~Soon Ong, and Bob Williamson.
\newblock Learning from corrupted binary labels via class-probability
  estimation.
\newblock In \emph{International {Conference} on {Machine} {Learning}}, pages
  125--134, 2015.

\bibitem[Menon et~al.(2019)Menon, Rawat, Reddi, and
  Kumar]{menon_multilabel_2019}
Aditya~K. Menon, Ankit~Singh Rawat, Sashank Reddi, and Sanjiv Kumar.
\newblock Multilabel reductions: what is my loss optimising?
\newblock \emph{Advances in Neural Information Processing Systems}, 32, 2019.
\newblock URL
  \url{https://papers.nips.cc/paper/2019/hash/da647c549dde572c2c5edc4f5bef039c-Abstract.html}.

\bibitem[Mikolov et~al.(2013)Mikolov, Chen, Corrado, and
  Dean]{mikolov2013efficient}
Tomas Mikolov, Kai Chen, Greg Corrado, and Jeffrey Dean.
\newblock Efficient estimation of word representations in vector space.
\newblock \emph{arXiv preprint arXiv:1301.3781}, 2013.

\bibitem[Mohri et~al.(2018)Mohri, Rostamizadeh, and
  Talwalkar]{mohri2018foundations}
Mehryar Mohri, Afshin Rostamizadeh, and Ameet Talwalkar.
\newblock \emph{Foundations of machine learning}.
\newblock MIT press, 2018.

\bibitem[Natarajan et~al.(2017)Natarajan, Dhillon, Ravikumar, and
  Tewari]{natarajan_cost-sensitive_2017}
Nagarajan Natarajan, Inderjit~S. Dhillon, Pradeep Ravikumar, and Ambuj Tewari.
\newblock Cost-sensitive learning with noisy labels.
\newblock \emph{The Journal of Machine Learning Research}, 18\penalty0
  (1):\penalty0 5666--5698, 2017.
\newblock Publisher: JMLR. org.

\bibitem[Partalas et~al.(2015)Partalas, Kosmopoulos, Baskiotis, Artieres,
  Paliouras, Gaussier, Androutsopoulos, Amini, and Galinari]{partalas2015lshtc}
Ioannis Partalas, Aris Kosmopoulos, Nicolas Baskiotis, Thierry Artieres, George
  Paliouras, Eric Gaussier, Ion Androutsopoulos, Massih-Reza Amini, and Patrick
  Galinari.
\newblock Lshtc: A benchmark for large-scale text classification.
\newblock \emph{arXiv preprint arXiv:1503.08581}, 2015.

\bibitem[Prabhu and Varma(2014)]{prabhu2014fastxml}
Yashoteja Prabhu and Manik Varma.
\newblock Fastxml: A fast, accurate and stable tree-classifier for extreme
  multi-label learning.
\newblock In \emph{KDD}, pages 263--272. ACM, 2014.

\bibitem[Qaraei et~al.(2021)Qaraei, Schultheis, Gupta, and
  Babbar]{qaraei_convex_2021}
Mohammadreza Qaraei, Erik Schultheis, Priyanshu Gupta, and Rohit Babbar.
\newblock Convex {Surrogates} for {Unbiased} {Loss} {Functions} in {Extreme}
  {Classification} {With} {Missing} {Labels}.
\newblock In \emph{Proceedings of the {Web} {Conference} 2021}, pages
  3711--3720, Ljubljana Slovenia, April 2021. ACM.
\newblock ISBN 978-1-4503-8312-7.
\newblock \doi{10.1145/3442381.3450139}.
\newblock URL \url{https://dl.acm.org/doi/10.1145/3442381.3450139}.

\bibitem[Shalev-Shwartz and Ben-David(2014)]{shalev2014understanding}
Shai Shalev-Shwartz and Shai Ben-David.
\newblock \emph{Understanding machine learning: From theory to algorithms}.
\newblock Cambridge university press, 2014.

\bibitem[Van~Rooyen and Williamson(2017)]{van_rooyen_theory_2017}
Brendan Van~Rooyen and Robert~C. Williamson.
\newblock A theory of learning with corrupted labels.
\newblock 18\penalty0 (1):\penalty0 8501--8550, 2017.
\newblock ISSN 1532-4435.

\bibitem[Wydmuch et~al.(2021)Wydmuch, Jasinska-Kobus, Babbar, and
  Dembczynski]{propensity-plt}
Marek Wydmuch, Kalina Jasinska-Kobus, Rohit Babbar, and Krzysztof Dembczynski.
\newblock \emph{Propensity-Scored Probabilistic Label Trees}, page 2252–2256.
\newblock Association for Computing Machinery, New York, NY, USA, 2021.
\newblock ISBN 9781450380379.
\newblock URL \url{https://doi.org/10.1145/3404835.3463084}.

\bibitem[You et~al.(2019)You, Zhang, Wang, Dai, Mamitsuka, and
  Zhu]{attentionxml}
Ronghui You, Zihan Zhang, Ziye Wang, Suyang Dai, Hiroshi Mamitsuka, and
  Shanfeng Zhu.
\newblock Attentionxml: Label tree-based attention-aware deep model for
  high-performance extreme multi-label text classification.
\newblock In H.~Wallach, H.~Larochelle, A.~Beygelzimer, F.~d\textquotesingle
  Alch\'{e}-Buc, E.~Fox, and R.~Garnett, editors, \emph{Advances in Neural
  Information Processing Systems}, volume~32. Curran Associates, Inc., 2019.
\newblock URL
  \url{https://proceedings.neurips.cc/paper/2019/file/9e6a921fbc428b5638b3986e365d4f21-Paper.pdf}.

\bibitem[Yu et~al.(2014)Yu, Jain, Kar, and Dhillon]{yu2014large}
Hsiang-Fu Yu, Prateek Jain, Purushottam Kar, and Inderjit Dhillon.
\newblock Large-scale multi-label learning with missing labels.
\newblock In \emph{International conference on machine learning}, pages
  593--601, 2014.

\bibitem[Zheng et~al.(2020)Zheng, Wu, Goswami, Goswami, Metaxas, and
  Chen]{zheng2020error}
Songzhu Zheng, Pengxiang Wu, Aman Goswami, Mayank Goswami, Dimitris Metaxas,
  and Chao Chen.
\newblock Error-bounded correction of noisy labels.
\newblock In \emph{International Conference on Machine Learning}, pages
  11447--11457. PMLR, 2020.

\end{thebibliography}

\appendix

\onecolumn

\section{Examples of Loss Functions}
\label{app:losses}

\subsection{Multilabel Losses}
\begin{remark}[Computational Complexity]
    \label{rem:cc}
    The computation required to calculate the unbiased estimate for a general multilabel loss according to \autoref{thm:efficient-multilabel} scales exponentially with the number of observed labels.
    For large label spaces, this is typically far less than the number of true labels \citep{Jain16}. If
    we assume this to be $O(\log(l))$ and have $f^*$ be computable in $O(l^k)$, 
    then we need computation on the order of
    \begin{equation}
        2^{O(\log(l))} \cdot \left( O(\log(l)) + O(l^k) \right) = O\left(l^{1+k}\right).
    \end{equation}
\end{remark}

\paragraph{Per-Example Recall}
\label{appendix:per-example-recall}
Applying the general solution from \autoref{thm:efficient-multilabel} to the
definition of per-example Recall as given in equation
\eqref{eq:def:recall} results in 
\begin{equation}
    \operatorname{Recall}(\vect{y}, \vect{\hat{y}}) = \left(\prod_{i \in \mathcal{I}(\vect{y})} \hspace{-1.5ex} \frac{p_i - 1}{p_i} \right) \cdot \sum_{\mathclap{\mathcal{J} \subset \mathcal{I}(\vect{y})}} |\mathcal{J}|^{-1} \! \left( \sum_{k \in \mathcal{J}} \hat{y}_k \right) \! \prod_{j \in \mathcal{J}} \left( p_j - 1\right)^{-1}.
\end{equation}

Note that the predictions $\vect{\hat{y}}$ usually are a sparse vector, \eg
when calculating recall@k, there are exactly $k$ nonzero entries. This may
allow for a slightly more efficient calculation. We
denote the set of correct predictions as $\mathcal{S}
\coloneqq \mathcal{I}(\vect{y}) \cap \mathcal{I}(\vect{\hat{y}})$ and the
missed labels as $\mathcal{T} = \mathcal{I}(\vect{y}) \setminus \mathcal{S}$,
such that $\mathcal{I}(\vect{y}) = \mathcal{S} \cup \mathcal{T}$ and the sum
over subsets of $\mathcal{I}(\vect{y})$ can be written as a nested sum for
$\mathcal{S}$ and $\mathcal{T}$. For convenience, we abbreviate the common
factor as $c(\vect{y}) \coloneqq \prod_{i \in \mathcal{I}(\vect{y})} \frac{p_i - 1}{p_i}$, and set $d(U) \coloneqq \prod_{j \in \mathcal{U}} \left( p_j - 1\right)$. 
This results in 
\begin{equation}
     \operatorname{Recall}(\vect{y}, \vect{\hat{y}}) 
    = c(\vect{y}) \sum_{\mathclap{\substack{\mathcal{U} \subset \mathcal{T} \\ \mathcal{V} \subset \mathcal{S}}}} \frac{|\mathcal{V}|}{|\mathcal{U}| + |\mathcal{V}|} \left( \prod_{j \in \mathcal{U}} \frac{1}{p_j - 1} \right) \left( \prod_{k \in \mathcal{V}} \frac{1}{p_k - 1} \right) 
    = c(\vect{y}) \sum_{\mathclap{\mathcal{V} \subset \mathcal{S}}} \frac{|\mathcal{V}|}{d(\mathcal{V})} \sum_{\mathclap{\mathcal{U} \subset \mathcal{T}}} \frac{d(\mathcal{U})^{-1} }{|\mathcal{U}| + |\mathcal{V}|}.
\end{equation}

The second sum is almost independent of the first: If the number of elements
in $\mathcal{V}$ were constant, we could change the nested sums into a product
of two single sums. Therefore, we collect terms based on the number of
elements in $\mathcal{V}$ and rearrange to recover the result given in the
main text
\begin{equation}
    \operatorname{Recall}(\vect{y}, \vect{\hat{y}}) =
      c(\vect{y}) \sum_{s=1}^{|\mathcal{S}|} \sum_{\substack{\mathcal{V} \subset \mathcal{S} \\|\mathcal{V}| = s}} \frac{s}{d(\mathcal{V})} \cdot \sum_{\mathclap{\mathcal{U} \subset \mathcal{T}}} \frac{d(\mathcal{U})^{-1} }{|\mathcal{U}| + s} 
    = c(\vect{y}) \sum_{s=1}^{|\mathcal{S}|} 
    \left(\sum_{{\mathcal{U} \subset \mathcal{T}}} \frac{d(\mathcal{U})^{-1} }{|\mathcal{U}| + s} \right) \sum_{\substack{\mathcal{V} \subset \mathcal{S} \\|\mathcal{V}| = s}} \frac{s}{d(\mathcal{V})}.
\end{equation}

\paragraph{Pairwise Loss Functions}
We can also consider loss functions of the form
\begin{equation}
    f^*(\vect{y}, \vect{\hat{y}}) = \sum_{i = 1}^l \sum_{j=i+1}^l g_{y_i, y_j}(\hat{y}_i, \hat{y}_j) = \sum_{a, b=0}^1 \sum_{i = 1}^l \sum_{j=i+1}^l \ind[y_i=a] \ind[y_j=b] g_{a, b}(\hat{y}_i, \hat{y}_j)
\end{equation}
for four given functions $\defmap{g_{ab}}{\predictspace \times
\predictspace}{\real}$. As before, we can rewrite the indicators as linear
functions $\ind{y=a} = y a + (1-y)(1-a)$ since the components of the label
vector $\vect{y}$ are either zero or one. This leads to the expression
\begin{multline}
    f^*(\vect{y}, \vect{\hat{y}}) = \sum_{a, b=0}^1 \sum_{i = 1}^l \sum_{j=i+1}^l [y_ia + (1-y_i)(1-a)][y_j b + (1-y_j)(1-b)] g_{a, b}(\hat{y}_i, \hat{y}_j) \\
    = \sum_{a, b=0}^1 \sum_{i = 1}^l \sum_{j=i+1}^l [y_i(2a-1) + 1 - a)][y_j (2b-1) + 1-b] g_{a, b}(\hat{y}_i, \hat{y}_j).  \label{eq:pairwise-loss}
\end{multline}

Instead of applying \autoref{thm:efficient-multilabel}, we can use that the
preceding equation is linear in each label separately (the summation limits
exclude $i=j$), so by the independence of the mask variables we can directly
write down
\begin{multline}
    f(\vect{y}, \vect{\hat{y}}) = \sum_{a, b=0}^1 \sum_{i = 1}^l \sum_{j=i+1}^l \left(\frac{y_i}{p_i}(2a-1) + 1 - a\right) \left(\frac{y_j}{p_j}(2b-1) + 1-b\right) g_{a, b}(\hat{y}_i, \hat{y}_j) \\
    = \sum_{i = 1}^l \sum_{j=i+1}^l p_i^{-1}p_j^{-1} \sum_{a, b=0}^1 \left((2a-1)y_i+ p_i(1 - a) \right) \left((2b-1)y_j + p_j(1-b)\right) g_{a, b}(\hat{y}_i, \hat{y}_j).
\end{multline}

An example is the Kendall-Tau loss \citep[p. 202]{shalev2014understanding}, which is used 
for ranking applications and counts how many pairs of labels are ranked differently in the
prediction than in the ground truth. Multilabel classification can be interpreted as a 
form of bipartite ranking, where no loss is incurred if both labels have the same ground-truth
independent of their predictions \citep{bucak2009efficient}. In that case, we have
\begin{align}
    g_{00}(p, q) &= 0 &         g_{10}(p, q) &= \ind[q>p] & \\
    g_{01}(p, q) &= \ind[p>q] & g_{11}(p, q) &= 0.
\end{align}
Using the zeros for equal ground-truth ranking, the unbiased estimate can be simplified to
\begin{equation}
    \operatorname{KT}(\vect{y}, \vect{\hat{y}}) 
    = \sum_{i = 1}^l \sum_{j=i+1}^l p_i^{-1}p_j^{-1} \left(p_i - y_i\right) y_j g_{01}(\hat{y}_i, \hat{y}_j) + y_i \left(p_j - y_j\right) g_{10}(\hat{y}_i, \hat{y}_j),
\end{equation}
where in a training application the discontinuous $g_{01}, g_{10}$ from above
would be replaced by a convex upper-bound. Note that, contrary to the general
solution as given by \autoref{thm:efficient-multilabel}, this equation
contains only the product of two inverses of the propensity. Thus, we can
expect that the increase in variance should be less severe than the one we
observed for recall, but more than we get for the binary losses. OF course,
due to the uniqueness \autoref{thm:multi-unique}, when plugging in the
definition of a pairwise loss, the general solution will be the same as the
one given here. The reason we did not take this approach is that it would be
more work to try and simplify the sum-over-subsets structure than it is to
derive the solution from scratch, though using the same principles, \ie the
independent mask and the linearization that is possible because the label set
is discrete.

\section{Theorems and Proofs}
\label{app:proofs}
\subsection{Unbiased Estimates for Binary Setting}
\thmbingrad*
\begin{proof}
    \label{proof:corbingrad}
    The independence of $\datapvar, \vect{W}$ and $M$ implies the independence
    of $\phi(\datapvar; \vect{W})$ and $M$, so the first equation
    follows directly from \autoref{thm:single_label}. For the second, 
    set $\dataspace^\prime = \dataspace \times \real^k$, $X^\prime = (X, \vect{W})$ and define 
    \begin{align}
        \defmap{g^*}{\labelspace \times \dataspace^\prime&}{\real^k} \\
        (y, (x, \vect{w})) &\mapsto \nabla_{\vect{w}} \ell^*(y, \phi(x, \vect{w})).
    \end{align}
    We can apply \autoref{thm:single_label} to $X^\prime$, where we set $g =
    \pswt(g^*)$. It remains to be shown that $g$ has the required structure,
    which follows from the linearity of the derivative through
    \begin{align}
        g(1, x, \vect{w}) &= p^{-1} \left(g^*(1, x, \vect{w}) + (p-1) g^*(0, x, \vect{w}) \right) \\                                       
                   &= p^{-1} \left(\nabla_{\vect{w}} \ell^*(1, \phi(x, \vect{w})) + (p-1) \nabla_{\vect{w}} \ell^*(0, \phi(x, \vect{w})) \right) \\       
                   &= \nabla_{\vect{w}} \left( p^{-1} \left( \ell^*(1, \phi(x, \vect{w})) + (p-1) \ell^*(0, \phi(x, \vect{w})) \right) \right) \\                   
                   &= \nabla_{\vect{w}} \ell(1, \phi(x, \vect{w})).                                                                        
    \end{align}
\end{proof}

\thmpsuniqueness*
\begin{proof}
\label{proof:uniqueness}
    The fact that $\pswt$ fulfills the condition \eqref{eq:unbiased-operator}
    follows from \autoref{thm:single_label}. Because $\expect[Y] = q$, the
    additional term has expectation zero.

    Let $\mathfrak{P}$ be another operator for which the condition holds, then
    \eqref{eq:unbiased-operator} is in particular also fulfilled for distributions
    of the form
    \begin{equation}
        X, Y^* \sim q^* \delta(x_1, 1) + (1-q^*) \delta(x_2, 0),
    \end{equation}
    where $\delta(x, y)$ denotes the Dirac measure of point $(x, y)$ and $q^*
    \coloneqq q/p$. These are valid because $q \stackrel{!}{=} \prob[Y=1] =
    \prob[Y^*=1, M=1] = p \cdot q^* = q$.

    Let us further write $f=\mathfrak{P}(f^*)$, $f^*(y, x) = y g^*(x) + h^*(x)$
    and $f(x, y) = y g(x) + h(x)$. Since $y \in \set{0,1}$, the decomposition into
    $g$ and $h$ is always possible.

    In this notation, we can explicitly calculate the expectations
    \begin{align}
         \expect[f^*(Y^*, X)] &= \expect[Y^* g^*(X) + h^*(X)] = q^*(g^*(x_1) + h^*(x_1)) + (1-q^*)h^*(x_2) \\ 
         \expect[f(Y, X)] &= \expect[M Y^* g(X) + h(X)] = q^*(p g(x_1) + h(x_1)) + (1-q^*)h(x_2).
    \end{align}
    By assumption on $\mathfrak{P}$ these two are equal:
    \begin{gather}
        q^*(g^*(x_1) + h^*(x_1)) + (1-q^*)h^*(x_2) = q^*(p g(x_1) + h(x_1)) + (1-q^*)h(x_2) \\ 
        q^*(g^*(x_1) + h^*(x_1) - h(x_1)) + (1-q^*)(h^*(x_2) - h(x_2)) = p \, q^* g(x_1). \label{eq:unq-proof:cond}
    \end{gather}
    Setting $x_1 = x_2 \eqqcolon x$ gives 
    \begin{equation}
        q^* (g^*(x) - p g(x)) = h(x) - h^*(x), \label{eq:unq-proof:h}
    \end{equation}
    which can be plugged back into \eqref{eq:unq-proof:cond}
    \begin{align}
        q^*(g^*(x_1) - q^* (g^*(x_1) - p g(x_1))) + (1-q^*)(q^* (g^*(x_2) - p g(x_2))) &= p q^* g(x_1) \\
        (1 - q^*) g^*(x_1) + (q^* - 1) p g(x_1) + (1-q^*)(g^*(x_2) - p g(x_2)) &= 0 \\
        g^*(x_1) - p g(x_1) + g^*(x_2) - p g(x_2) &= 0.
    \end{align}
    Since this equation holds for all $x_1, x_2 \in \mathcal{X}$, it determines $g$ up to a constant shift.
    Let $x_0 \in \mathcal{X}$ be fixed and denote $c=g^*(x_0)/p -  g(x_0)$, then for arbitrary $x \in \mathcal{X}$
    \begin{equation}
         g(x) = g^*(x)/p + c.
    \end{equation}
    Putting this back into \eqref{eq:unq-proof:h} gives
    \begin{equation}
        h(x) = h^*(x)- p c q^* = h^*(x) - q c.
    \end{equation}
    Combining these two expressions shows the claim
    \begin{equation}
        \mathfrak{P}(f^*)(y, x) = y g(x) + h(x) = \frac{y g^*(x)}{p} + y c + h^*(x) - q c = \frac{y}{p} g^*(x) + h^*(x) + (y - q) c.
    \end{equation}
\end{proof}

\subsection{Generalization Bound}
In this section we present a proof for \autoref{thm:genbound}. To that end, we first proof
some helper results as presented below:
\begin{theorem}
    \label{thm:cbound-binary}
    Let $\functionclass \subset \allmaps{\dataspace}{\predictspace}$ be a
    function class, and $\defmap{f}{\labelspace \times \predictspace}{[a, b]}$
    be a bounded, $\rho$-Lipschitz continuous (in the second argument) 
    function. Denote $\tilde{f} \coloneqq \pswt f$. Given a sample of $n$ noisy training points, it holds with probability
    of at least $1-\delta$ that
    \begin{align}
         \sup_{h \in \functionclass} \left( \oerisk_{\tilde{f}}[h] - \orisk_{\tilde{f}}[h] \right) &\leq \frac{4-2p}{p} \rho \rademacher_n(\functionclass) + \frac{(2 - p)(b-a)}{p} \sqrt{\frac{\log(1/\delta)}{2n}} \\
         \sup_{h \in \functionclass} \left( \orisk_{\tilde{f}}[h] - \oerisk_{\tilde{f}}[h] \right) &\leq \frac{4-2p}{p} \rho \rademacher_n(\functionclass) + \frac{(2 - p)(b-a)}{p} \sqrt{\frac{\log(1/\delta)}{2n}}.
    \end{align}
\end{theorem}
\begin{proof}
    First, we determine the Lipschitz-constant of $\tilde{f}$. For $y=0$, it is the same as that of $f$, so we only need to consider the $y=1$ case.
    \begin{align}
        \abs{\tilde{f}(1, x_1) - \tilde{f}(1, x_2)} &= \frac{f(1, x_1) + (p-1)f(0, x_1) - f(1, x_2) - (p-1) f(0, x_2)}{p} \\
        &\leq \frac{1}{p} \abs{f(1, x_1) - f(1, x_2)} + \frac{1-p}{p} \abs{f(0, x_1) - f(0, x_2)} \label{eq:calc-rho-p-lt-1}\\ 
        &\leq \left( \frac{1}{p} + \frac{1-p}{p} \right) \rho \| x_1 - x_2 \|.
    \end{align}
    In \eqref{eq:calc-rho-p-lt-1} we made use of the fact that $0 < p \leq
    1$. This also implies that $\frac{2-p}{p} \geq 1$, and thus the Lipschitz
    constant of $\tilde{f}$ is given by $\frac{2-p}{p} \rho$.

    Next we calculate the range of $\tilde{f}$. We have $\forall x \in \dataspace$, that
    \begin{equation}
        \tilde{a} \coloneqq \frac{a + (p-1)b}{p} \leq a \leq \tilde{f}(1, x) \leq b \leq \frac{b + (p-1)a}{p} \eqqcolon \tilde{b}.
    \end{equation}
    Here, the first and last inequality follow from $p \in (0, 1]$ and $a \leq b$. Define $c \coloneqq \tilde{b} - \tilde{a}$.

    Finally, we can construct a function $\defmap{f_{01}}{\labelspace \times \predictspace}{[0, 1]}$ by the affine transformation
    $f_{01} = c^{-1} (\tilde{f} - \tilde{a})$ such that
    \begin{equation}
         \orisk_{\tilde{f}}[h] - \oerisk_{\tilde{f}}[h] = c \left( \orisk_{f_{01}}[h] - \oerisk_{f_{01}}[h] \right).
    \end{equation}

    The right hand side can be bounded with probability $1-\delta$ using \citet[Theorem 3.3]{mohri2018foundations} by
    \begin{equation}
        \orisk_{f_{01}}[h] - \oerisk_{f_{01}}[h] \leq \rademacher_n(f_{01} \circ \functionclass) + \sqrt{\frac{\log(1/\delta)}{2n}}.
    \end{equation}
    As the Lipschitz-constant of $f_{01}$ is $c^{-1} p^{-1}(2-p)$, by the contraction lemma \citep[Lemma 26.9]{shalev2014understanding} we have
    \begin{equation}
        \rademacher_n(f_{01} \circ \functionclass) \leq c^{-1} \frac{2-p}{p} \rho \rademacher_n(\functionclass).
    \end{equation}
    Thus with probability $1-\delta$ and $\forall h \in \functionclass$
    \begin{align}
        \orisk_{\tilde{f}}[h] - \oerisk_{\tilde{f}}[h] &\leq c c^{-1} \frac{2-p}{p} \rho \rademacher_n(\functionclass) + c \sqrt{\frac{\log(1/\delta)}{2n}} \\
        &= \frac{2-p}{p} \rho \rademacher_n(\functionclass) + \frac{(2-p)(b-a)}{p} \sqrt{\frac{\log(1/\delta)}{2n}}
    \end{align}

    The second bound follows by replacing $f$ with $-f$.
\end{proof}    

This result is very similar to \citet[Lemma 8]{natarajan_cost-sensitive_2017}.
However, that theorem is missing a scaling factor with the range of the loss
function, as argued below. For reference, the original statement of the theorem is
\begin{theorem}[{\citet[Lemma 8]{natarajan_cost-sensitive_2017}}]
    \label{thm:_natarajan}
    Let $l(t, y)$ be $L$-Lipschitz in $t$ (for every $y$). Then, for any $\alpha \in (0, 1)$,
    with probability at least $1-\delta$,
    \begin{equation}
        \max_{f \in \mathcal{F}} \abs{\hat{R}_{\tilde{l}_\alpha}(f) - R_{\tilde{l}_\alpha, D_\rho}(f)} \leq 2 L_\rho \rademacher_n(\mathcal{F}) + \mathbf{} \sqrt{\frac{\log(1/\delta)}{2n}},
    \end{equation}
    where $L_\rho \leq 2L/(1 - \rho_{+1} - \rho_{-1})$ is the Lipschitz constant of $\tilde{l}_{\alpha}$.
\end{theorem}
In the first step of the proof, they invoke a \textquote{Basic Rademacher bound between risks and empirical risks} that states
\begin{equation}
    \max_{f \in \mathcal{F}} \abs{\hat{R}_{\tilde{l}_\alpha}(f) - R_{\tilde{l}_\alpha, D_\rho}(f)} \leq 2 \rademacher_n(\tilde{l}_{\alpha} \circ \mathcal{F}) + \sqrt{\frac{\log(1/\delta)}{2n}}
\end{equation}
However, such a bound either requires the range of $\tilde{l}_{\alpha}$ to be
a subset of $[0, 1]$ \citep[Thm 3.3]{mohri2018foundations}, or introduces an
additional factor in front of the square root term as in \citet[Thm
26.5]{shalev2014understanding}. Also, they are using a two-sided bound instead
of a one-sided one as in the two cited theorems, which means that $\delta$
needs to be replaced with $\delta/2$ because the square-root term comes from
an application of Mc.\ Diamids inequality.


\begin{lemma}
    \label{lemma:lossriskdiff}
    For any $f \in \allmaps{\labelspace \times \dataspace}{\real}$, $p \in (0, 1]$ and $\expect[\obslabels] = q$, it holds
    \begin{equation}
        \expect[\abs{\pswt(f)(\obslabels, \datapvar) - f(\obslabels, \datapvar)}] \leq q \frac{1 - p}{p} m \quad \text{with}\quad m \coloneqq \sup_{x \in \dataspace}(\abs{f(1, x) - f(0, x)}).
    \end{equation}
\end{lemma}
\begin{proof}
    For notational convenience denote $U \coloneqq f(1, \datapvar)$ and $V \coloneqq f(0, \datapvar)$. Substituting $\pswt(f)$, difference can be simplified to
    \begin{align*}
        \MoveEqLeft
        \pswt(f)(Y, X) - f(Y, X)  \\
        &= \obslabels p^{-1} \left( f(1, \datapvar) + (p-1) f(0, \datapvar) \right) + (1 - \obslabels)f(0, \datapvar) - f(\obslabels, \datapvar)  \\
        &= \obslabels p^{-1} \left( U + (p-1) V \right) + (1 - \obslabels)V - \obslabels U - (1-\obslabels)V  \\
        &= (p^{-1} - 1) (\obslabels U)  + \frac{p-1}{p} \obslabels V = \frac{1-p}{p} (\obslabels (U - V)).
    \end{align*}
    As by definition of $m$ it holds that $\abs{U-V} \leq m$, the expectation is bounded by
    \begin{equation}
        \expect[\abs{\pswt(f)(Y, X) - f(Y, X)}]  = \frac{1-p}{p} \expect[\abs{\obslabels (U - V)}] \leq \frac{1-p}{p} \expect[\obslabels m] \leq \frac{1-p}{p} q m.
    \end{equation}
\end{proof}

\thmgenbound*
\begin{proof}
    \label{proof:thmgenbound}
    Let $\epsilon > 0$ and choose $h^{\prime}$ such that
    \begin{equation}wever
        \trisk_{f^*}[h^\prime] \leq r^{\prime} + \epsilon.
    \end{equation}
    From this follows $\hat{r} - r^\prime \leq \trisk_{f^*}[\hat{h}] - \trisk_{f^*}[h^\prime] + \epsilon$ and $\tilde{r} - r^\prime \leq \trisk_{f^*}[\tilde{h}] - \trisk_{f^*}[h^\prime] + \epsilon$.

    We can apply \autoref{thm:cbound-binary} to the function class $\{h^\prime\}$
    using that $\rademacher_n(\{h^\prime = 0\}) = 0$ to get with probability
    $1-\delta/2$
    \begin{equation}
        \oerisk_{f}[h^\prime] - \orisk_f[h^\prime] \leq \frac{(2 - p)(b-a)}{p} \sqrt{\frac{\log(2/\delta)}{2n}}.
    \end{equation}

    For the first inequality, we can use the unbiasedness of $f$, and the near optimality of $\hat{h}$ regarding $\oerisk_f$ to bound
    \begin{align*}
        \trisk_{f^*}[\hat{h}] - \trisk_{f^*}[h^\prime] &= \orisk_f[\hat{h}] - \orisk_f[h^\prime] \tag*{(unbiasedness)}\\
        &= \orisk_f[\hat{h}] - \oerisk_f[\hat{h}] + \oerisk_f[\hat{h}] - \oerisk_f[h^\prime] + \oerisk_f[h^\prime]- \orisk_f[h^\prime] \\
        &\leq \orisk_f[\hat{h}] - \oerisk_f[\hat{h}] + \oerisk_f[h^\prime]- \orisk_f[h^\prime] \tag*{(optimality)}\\
        & \leq \sup_{h \in \mathcal{H}} \left( \orisk_f[h] - \oerisk_f[h] \right) + \oerisk_f[h^\prime]- \orisk_f[h^\prime].
    \end{align*}
    Applying a union bound to the remaining two terms, with probability
    $1-\delta$
    \begin{align}
        \hat{r} - r^\prime &\leq \epsilon + \frac{4-2p}{p} \rademacher(\mathcal{H}) + \frac{(2 - p)(b-a)}{p} \sqrt{\frac{\log(2/\delta)}{2n}} + \frac{(2 - p)(b-a)}{p} \sqrt{\frac{\log(2/\delta)}{2n}} \nonumber \\
        &= \epsilon + \frac{4-2p}{p} \left( \rademacher_n(\mathcal{H}) + (b-a) \sqrt{\frac{\log(2/\delta)}{2n}} \right).
    \end{align}
    With $\epsilon \rightarrow 0$ the first claim follows.

    For $\tilde{h}$ we can decompose the risk difference into
    \begin{equation}
        \trisk_{f^*}[\tilde{h}] - \trisk_{f^*}[h^\prime] = 
        \underbrace{\trisk_{f^*}[\tilde{h}] - \oerisk_{f^*}[\tilde{h}]}_{a} + \underbrace{\oerisk_{f^*}[\tilde{h}] - \oerisk_{f^*}[h^\prime]}_{b} + \underbrace{\oerisk_{f^*}[h^\prime] - \trisk_{f^*}[h^\prime]}_{c}
    \end{equation}
    and look at the contributions separately. Because $\tilde{h}$ is an ERM, we have
    \begin{equation}
        b = \oerisk_{f^*}[\tilde{h}] - \oerisk_{f^*}[h^\prime] \leq 0.
    \end{equation}
    Further, using the unbiasedness and applying \autoref{lemma:lossriskdiff} to the function $(x, y) \mapsto f(y, \tilde{h}(x))$
    \begin{align}
        a = \trisk_{f^*}[\tilde{h}] - \oerisk_{f^*}[\tilde{h}] &= \orisk_{f}[\tilde{h}] - \oerisk_{f^*}[\tilde{h}] \tag*{(unbiasedness)}\\
        &= \orisk_{f}[\tilde{h}] - \orisk_{f^*}[\tilde{h}] + \orisk_{f^*}[\tilde{h}] - \oerisk_{f^*}[\tilde{h}]  \\
        &\leq q \frac{1 - p}{p} m + \orisk_{f^*}[\tilde{h}] - \oerisk_{f^*}[\tilde{h}] \tag*{(\autoref{lemma:lossriskdiff})}\\
        &\leq q \frac{1 - p}{p} m + \sup_{h \in \mathcal{H}} \left( \orisk_{f^*}[h] - \oerisk_{f^*}[h] \right)
    \end{align}
    We can apply \autoref{thm:cbound-binary} with $p=1$ to get corresponding bounds for $f^* = \mathfrak{P}_0 f^*$ so
    that with probability $1-\delta/2$ each
    \begin{align}
        \sup_{h \in \mathcal{H}} \left( \orisk_{f^*}[h] - \oerisk_{f^*}[h] \right) &\leq \frac{4-2}{1} \rademacher_n(\mathcal{H}) + \frac{(2 - 1)(b-a)}{1} \sqrt{\frac{\log(2/\delta)}{2n}}\\
        c = \oerisk_{f^*}[h^\prime] - \trisk_{f^*}[h^\prime] &\leq \frac{(2 - 1)(b-a)}{1} \sqrt{\frac{\log(2/\delta)}{2n}}.
    \end{align}
    Thus, by union bound, with probability $1-\delta$, it holds that
    \begin{equation}
        \trisk_{f^*}[\tilde{h}] - \trisk_{f^*}[h^\prime] \leq q \frac{1 - p}{p} m + 2 \rademacher_n(\mathcal{H}) + 2(b-a) \sqrt{\frac{\log(2/\delta)}{2n}} + \epsilon.
    \end{equation}
    Letting $\epsilon \rightarrow 0$ proves the claim. 
\end{proof}

\subsection{Unbiased Estimates for Multilabel Setting}

\thmmultilabel*
\begin{proof}
    \label{proof:multilabel}
    As in the binary case, we can use the finiteness of $\set{0,1}^l$ to write $f^*$ as
    \begin{align}
        f^*(y, x) &= \sum_{\mathcal{I} \subset [l]} \ind[y=\indvect{\mathcal{I}}] f^*(\indvect{\mathcal{I}}, x) \\
        &= \sum_{\mathcal{I} \subset [l]} f^*(\indvect{\mathcal{I}}, x) \left( \prod_{i \in \mathcal{I}} y_i \right) \prod_{j \in \overline{\mathcal{I}}} \left(1-y_j \right) \label{eq:multilabel:decomp}
    \end{align}
    by rewriting the indicator function as products of $y_i$ and $1-y_j$.

    First, we show the unbiasedness for the expression
    \begin{equation}
        \tilde{f}(\vect{y}, x) \coloneqq \sum_{\mathcal{I} \subset [l]} f^*(\indvect{\mathcal{I}}, x) \left( \prod_{i \in \mathcal{I}} \frac{y_i}{p_i} \right) \prod_{j \in \overline{\mathcal{I}}} \left(1-\frac{y_j}{p_j}\right).
        \label{eq:ps-multilabel-large}
    \end{equation}
    Later, we will show that this is in fact equivalent to \eqref{eq:ps-multilabel-eff}. Using the linearity of the expectation we can explicitly calculate
    \begin{equation}
        \expect[\tilde{f}(\vect{Y}, X)] = \sum_{\mathcal{I} \subset [l]} \expect[f^*(\indvect{\mathcal{I}}, X) \left( \prod_{i \in \mathcal{I}} \frac{M_i Y^*_i}{p_i} \right) \prod_{j \in \overline{\mathcal{I}}} \left(1-\frac{M_j Y^*_j}{p_j}\right) ] \\
    \end{equation}
    For a fixed subset $\mathcal{I}$ we can rewrite
    \begin{equation}
        f^*(\indvect{\mathcal{I}}, X) \left( \prod_{i \in \mathcal{I}} \frac{M_i Y^*_i}{p_i} \right) \prod_{j \in \overline{\mathcal{I}}} \left(1-\frac{M_j Y^*_j}{p_j}\right) = \sum_{\mathcal{J} \subset [l]} f^*(\indvect{\mathcal{I}}, X) \prod_{i \in \mathcal{J}}\alpha_{\mathcal{J}} {M_i Y^*_i}
    \end{equation}
    for appropriately chosen coefficients $\alpha_{\mathcal{J}} \in \real$. Using the independence of $\vect{M}$, it follows that
    \begin{align}
        \MoveEqLeft
        \expect[f^*(\indvect{\mathcal{I}}, X) \prod_{i \in \mathcal{J}}\alpha_{\mathcal{J}} {M_i Y^*_i}] = \expect[\prod_{i \in \mathcal{J}} M_i] \expect[f^*(\indvect{\mathcal{I}}, X) \prod_{j \in \mathcal{J}}\alpha_{\mathcal{J}} {Y^*_j}] \\ &= \left(\prod_{i \in \mathcal{J}} q_i \right) \expect[f^*(\indvect{\mathcal{I}}, X) \prod_{j \in \mathcal{J}}\alpha_{\mathcal{J}} {Y^*_j}] = \expect[f^*(\indvect{\mathcal{I}}, X) \prod_{j \in \mathcal{J}}\alpha_{\mathcal{J}} {p_j Y^*_j}].
    \end{align}
    Therefore, we can replace all occurrences of $M_i$ in the expectation with $p_i$ and compare the result to \eqref{eq:multilabel:decomp}
    \begin{align}
        \expect[\tilde{f}(\vect{Y}, X)] &= \sum_{\mathcal{I} \subset [l]} \expect[f^*(\indvect{\mathcal{I}}, X) \left( \prod_{i \in \mathcal{I}} \frac{p_i Y^*_i}{p_i} \right) \prod_{j \in \overline{\mathcal{I}}} \left(1-\frac{p_j Y^*_j}{p_j}\right) ] \\
        &= \sum_{\mathcal{I} \subset [l]} \expect[f^*(\indvect{\mathcal{I}}, X) \left( \prod_{i \in \mathcal{I}} Y^*_i \right) \prod_{j \in \overline{\mathcal{I}}} \left(1-Y^*_j\right)] \\
         &= \expect[f^*(\vect{Y}^*, X)].
    \end{align}

    Now we show that $\tilde{f} = f$: 
    For any $\mathcal{J} \not\subset \mathcal{I}(\vect{y})$ there is a
    $j \in \mathcal{J}$ such that $y_j = 0$, so the contribution of that summand is zero. Therefore
    \begin{align}
        \tilde{f}(\vect{y}, x) &= \sum_{\mathclap{\mathcal{J} \subset \mathcal{I}(\vect{y})}} f^*(\indvect{\mathcal{J}}, x) \left( \prod_{i \in \mathcal{J}} \frac{1}{p_i} \right) \prod_{j \in \overline{\mathcal{J}}} \left(1-\frac{y_j}{p_j}\right) .
    \shortintertext{Now, for every $j \in \overline{\mathcal{I}(\vect{y})}$ we know that $y_j=0$, so we can simplify further}
    &=  \sum_{\mathclap{\mathcal{J} \subset \mathcal{I}(\vect{y})}} f^*(\indvect{\mathcal{J}}, x) \left( \prod_{i \in \mathcal{J}} \frac{1}{p_i} \right) \prod_{k \in \mathcal{I}(\vect{y}) \setminus \mathcal{J}}\left(1-\frac{1}{p_k} \right).
    \end{align}
    Finally, note that
    \begin{equation}
        \left( \prod_{i \in \mathcal{J}} \frac{1}{p_i} \right) \prod_{k \in \mathcal{I}(\vect{y}) \setminus \mathcal{J}}\left(1-\frac{1}{p_k} \right) = 
         \left( \prod_{i \in \mathcal{J}} \frac{1}{p_i} \left(\frac{p_i - 1}{p_i} \right)^{-1} \right) \prod_{k \in \mathcal{I}(\vect{y})}\left(\frac{p_k-1}{p_k} \right),
    \end{equation}
    which proves the statement.
\end{proof}

\thmmultilabelunique*
\begin{proof}
\label{proof:multi-uniqueness}
    Let $f^*$ be an arbitrary function \defmap{f^*}{\labelspace \times \dataspace}{\real}.
    We need to show that 
    \begin{equation}        
    \tilde{f} \coloneqq \mathfrak{P}(f^*) = \mathfrak{P}_{\vect{p}}(f^*) \eqqcolon f. 
    \end{equation}
    Since \eqref{eq:multi-unique-condition} needs to work for all possible distributions of
    $\datapvar$ and $\vect{Y}^*$, it needs to work in particular also for
    $\prob[\datapvar=x, \vect{\truelabels}=\vect{y}^*] = 1$. 
    Since $\vect{\obslabels}$ can take
    only finitely many states, we can decompose $\mathfrak{P}(f^*)$ into a sum
    over these states. Since $f$ is known to fulfill \eqref{eq:multi-unique-condition}, the equation for
    $\tilde{f}$ becomes
    \begin{equation}
        0 = \expect[\tilde{f}(\vect{\obslabels}, \datapvar) - f(\vect{\obslabels}, \datapvar)] = \sum_{\vect{y} \in 2^{\labelspace}} \expect[(\tilde{f}(x, \vect{y}) - f(x, \vect{y})) \ind[\vect{y} = \obslabels]] = \sum_{\vect{y}^\prime \in 2^{\labelspace}} (\tilde{f}(x, \vect{y}^\prime) - f(x, \vect{y}^\prime)) \prob[\vect{y}^\prime = \obslabels] 
        \label{eq:proof-unqiue:unbiasedness-eq}
    \end{equation}

    Next we show that for all $\vect{y} \in \labelspace$ it holds that
    $\tilde{f}(x, \vect{y}) - f(x, \vect{y}) = 0$ via induction. First, assume
    that $\vect{y}^* = \vect{0}$. Then $\prob[\vect{y}^\prime = \obslabels] =
    \ind[\vect{y}^\prime = \vect{0}]$, and \eqref{eq:proof-unqiue:unbiasedness-eq} simplifies to
    \begin{equation}
        0 = \tilde{f}(x, \vect{0}) - f(x, \vect{0}),
    \end{equation}
    and therefore $\tilde{f}$ and $f$ are equal for this $\vect{y}^*$. 

    Next, assume that $\tilde{f}(\vect{y}^\prime, x) = f(\vect{y}^\prime, x)$ for all $\vect{y}^\prime$
    that have at most $m$ nonzero entries. Let $\vect{y}$ be a vector with $m+1$ nonzero entries, then
    \eqref{eq:proof-unqiue:unbiasedness-eq} can be written as
    \begin{multline}
        0 = \sum_{\smash{\substack{\vspace{2ex}\\\vect{y}^\prime \in 2^{\labelspace} \\ \abs{\vect{y}^\prime} \leq m}}} (\tilde{f}(x, \vect{y}^\prime) - f(x, \vect{y}^\prime)) \prob[\vect{y}^\prime = \obslabels] + (\tilde{f}(x, \vect{y}^*) - f(x, \vect{y}^*)) \prob[\vect{y}^* = \obslabels] \\ = (\tilde{f}(x, \vect{y}^*) - f(x, \vect{y}^*)) \prob[\vect{y}^* = \obslabels].
    \end{multline}
    Here, the first sum vanishes because all the summands are using a vector $\vect{y}^\prime$ with at most $m$ elements. Because we 
    assume that all propensities are nonzero, the $\prob[\vect{y}^* = \obslabels]$ factor is nonzero, which implies that 
    $\tilde{f}(x, \vect{y}^*) - f(x, \vect{y}^*)$ has to be zero.

    Therefore, $\tilde{f}$ and $f$ are identical. Since this holds for all $x \in \dataspace$, the operators $\mathfrak{P}$ and $\mathfrak{P}_{\vect{p}}$ have
    to be identical.
\end{proof}

\section{Experiments}
\subsection{Models used for PsRecall estimation}
\label{ssec:app:pseval}
The networks used to generate the results in \autoref{tab:psrec} were trained
using the DiSMEC \citep{dismec} algorithm, with the loss function being either
the squared-hinge-loss (VN) or a squared-hinge-loss based convex surrogate of
the unbiased estimate of the 0-1 loss as described in \citet{qaraei_convex_2021}
The datasets have been taken from the Extreme Classification Repository
\citep{repo}, and preprocessed by doing a tf-idf transformation.

We can still calculate an unbiased estimate in these cases, at the cost at
even further increase in variance: For these examples, we artificially
generate subsample ground truths where even more labels have gone missing, and
adjust the propensity accordingly. In practice, our implementation divides all
propensities by two and drops each ground truth label with a chance of 50\% if
the ground truth has more than 25 labels. To reduce the increase in variance
slightly, we average 100 subsamplings. This technique can be used recursively
if the resulting subsample still has too many labels.

The final calculation of the PsRecall values as given in \autoref{tab:psrec}
include outlier filtering. For the \texttt{AmazonCat-13K} dataset the
corresponding histograms are shown in \autoref{app:fig:histograms}.


\begin{figure}
    \includegraphics[width=0.5\linewidth]{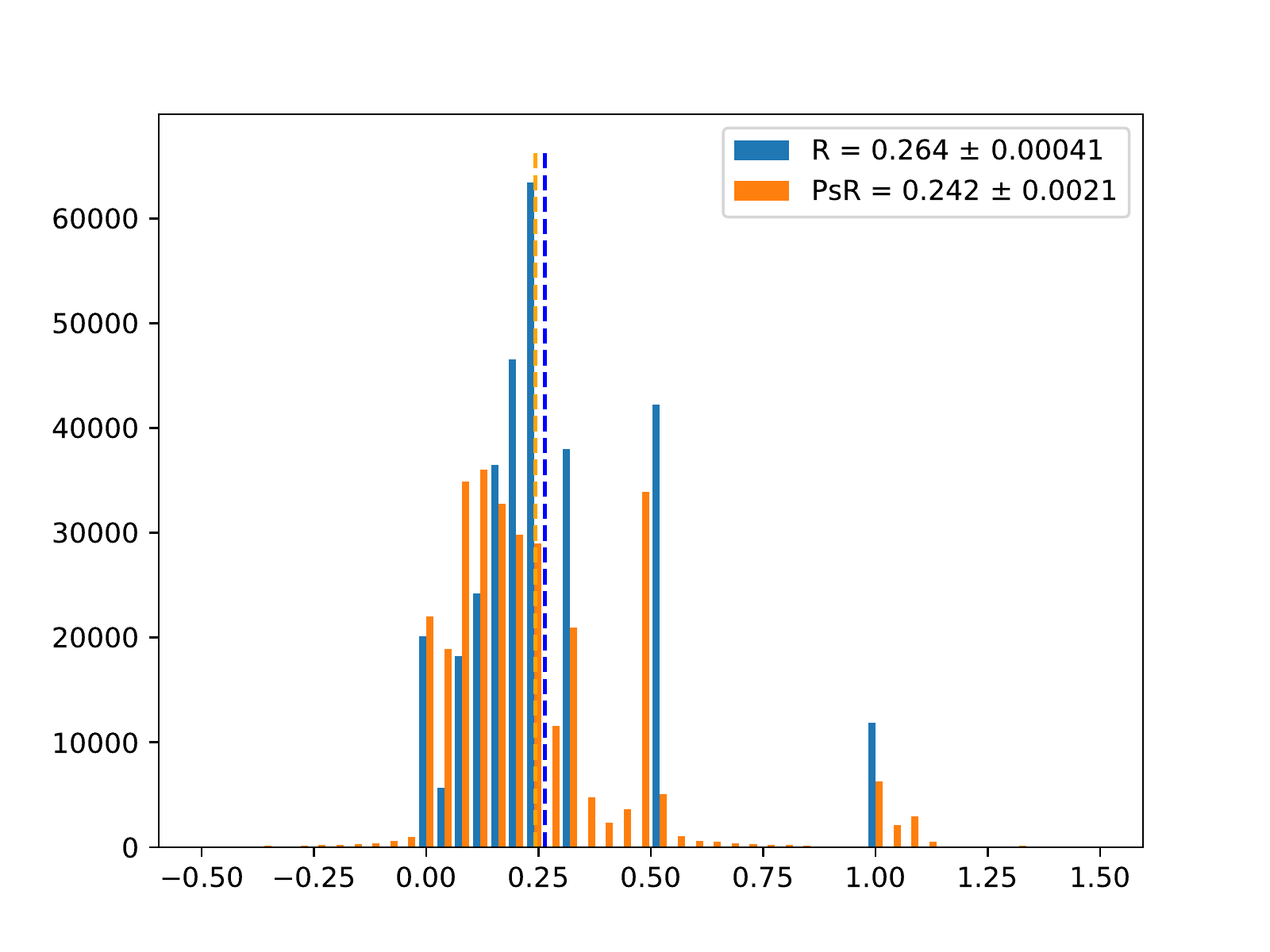}
    \includegraphics[width=0.5\linewidth]{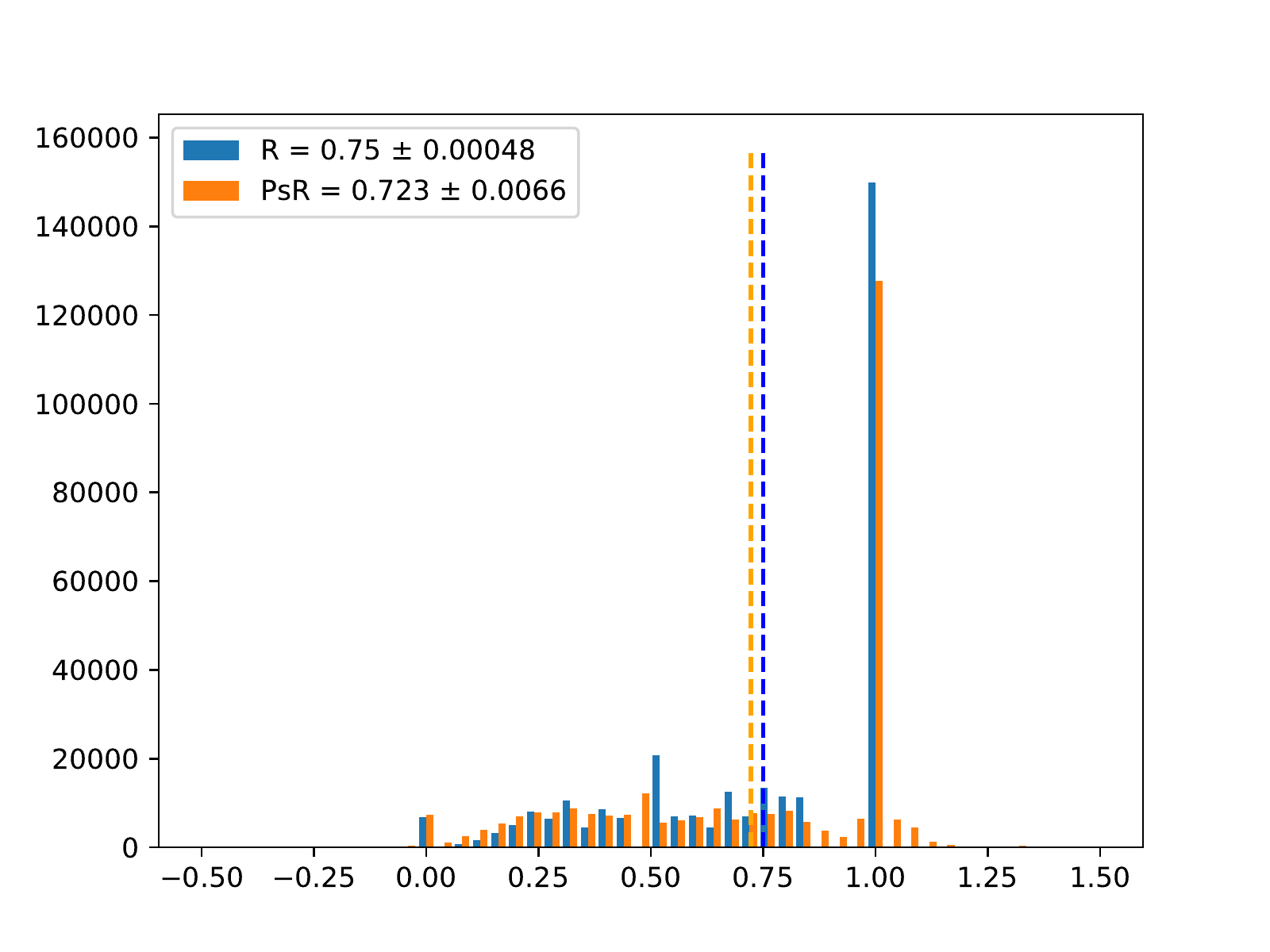}
    \caption{Histograms of propensity-scored (PsR) and vanilla recall (R)
    for top-1 (left) and top-5 (right) predictions for the
    \texttt{AmazonCat-13K} dataset for a DiSMEC model trained with a convex
    surrogate of the propensity-scored 0-1 loss. The y-axis denotes the number
    of examples in the test set for which the estimate falls into the
    corresponding bin. The errors have been calculated by boostrapping a 95 \%
    confidence interval. The dashed vertical lines denote the mean.}
    \label{app:fig:histograms}
\end{figure}


\subsection{Normalized BCE}
\label{seq:supp-norm-bce}
In the main text, we claim that the low performance of unbiased training
with normalized BCE in the high-regularization regime is due to bad local
minima. Here we present supporting evidence.

\begin{figure}[tbp]
    \centering
    \input{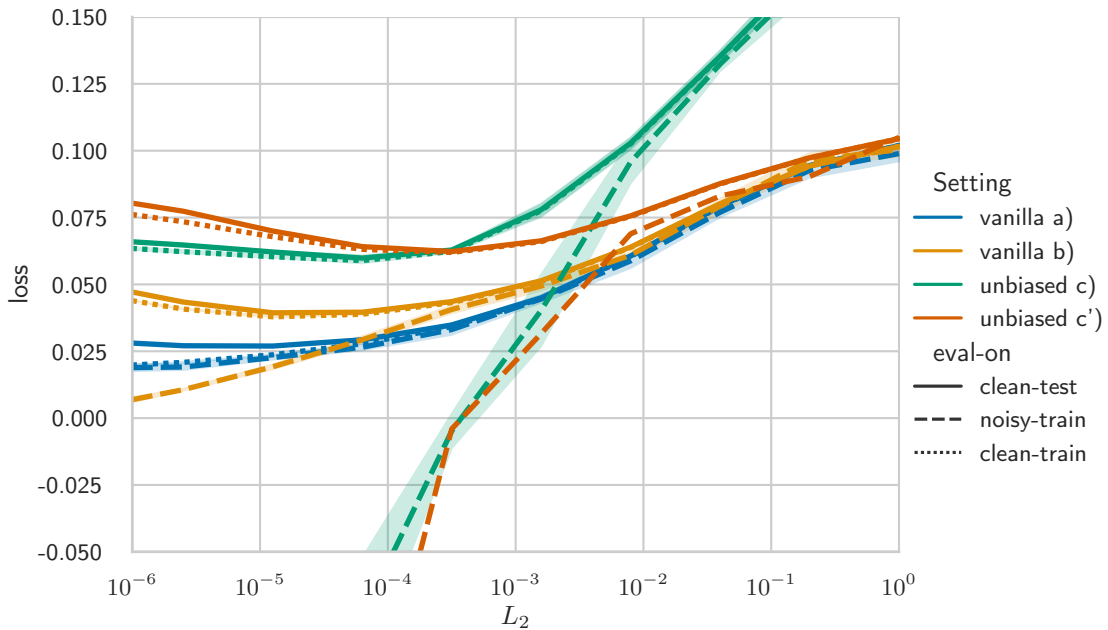}
    \caption{Normalized Binary cross-entropy for different regularization
    strengths, evaluated on noisy training data, clean training data, and
    clean test data. Setting c') corresponds to unbiased training with vanilla
    pre-training.}
    \label{fig:bce-norm-overfitting}
\end{figure}

As shown in \autoref{fig:bce-norm-overfitting}, the low performance is caused
by high loss on training data (dashed lines) as opposed to generalization.
However, for training with vanilla loss b), the loss on noisy training data
(i.e. the unbiased estimate computed on noisy training data) is much lower
than for training with the unbiased loss c), where we directly try to optimize
this quantity. There are two possible explanations for this behaviour: Either
the switch to the unbiased loss shifts the balance of loss to regularizer such
that the regularizer attains much more weight overall, and thus results in
less optimization for the loss, or the optimization procedure gets stuck in a
local minimum. 

These two causes can be distinguished by an experiment where the initial
weights are chosen by pre-training with vanilla loss for ten epochs. If the
reason for the bad performance were the regularization trade-off, then the
following 20 epochs of unbiased training would increase the training loss
until the trade-off is reached. On the other hand, if the reason were local
minima, then starting out close to a known "good" minimum would result in
converging to that minimum.

\autoref{fig:bce-norm-overfitting} exhibits the behaviour of the second case,
thus validating the local-minima hypothesis.

\end{document}